\documentclass[twoside]{article}

\usepackage[accepted]{aistats2025}

\usepackage[ruled,vlined]{algorithm2e}
\usepackage{enumitem}
\usepackage{amsmath,amsthm,amssymb,mathtools}
\usepackage{xcolor}
\usepackage[citecolor=blue,colorlinks=true]{hyperref}
\usepackage{cleveref}
\usepackage{thm-restate}

\newtheorem{theorem}{Theorem}

\newtheorem{lemma}[theorem]{Lemma}
\newtheorem{corollary}[theorem]{Corollary}
\newtheorem{example}[theorem]{Example}

\newtheorem{definition}[theorem]{Definition}

\newtheorem{remark}[theorem]{Remark}
\newtheorem{fact}[theorem]{Fact}

\newcommand{\eqdef}{\stackrel{\mbox{\scriptsize \rm def}}{=}}
\renewcommand{\leq}{\leqslant}
\renewcommand{\enspace}{\,}
\renewcommand{\geq}{\geqslant}
\renewcommand{\phi}{\varphi}
\renewcommand{\epsilon}{\varepsilon}
\DeclarePairedDelimiter\ceil{\lceil}{\rceil}
\newcommand{\ilog}[2]{\ceil{\log_{#1}(#2)}}
\newcommand{\R}{\mathbb{R}}
\newcommand{\E}{\mathbb{E}}

\newcommand{\cA}{\mathcal{A}}
\newcommand{\cH}{\mathcal{H}}

\newcommand{\cS}{\mathcal{S}}

\newcommand{\cO}{\mathcal{O}}
\newcommand{\cG}{\mathcal{G}}
\newcommand{\cE}{\mathcal{E}}

\newcommand{\cC}{\mathcal{C}}
\newcommand{\cX}{\mathcal{X}}

\newcommand{\ind}{\mathbb{I}}
\newcommand{\tcO}{\widetilde{\mathcal{O}}}
\newcommand{\bP}{\boldsymbol{P}}
\newcommand{\br}{\boldsymbol{r}}
\newcommand{\bb}{\boldsymbol{b}}
\newcommand{\bg}{\boldsymbol{g}}
\newcommand{\bw}{\boldsymbol{w}}
\newcommand{\hatbP}{\widehat{\boldsymbol{P}}}
\newcommand{\hatP}{\widehat{P}}

\newcommand{\bQ}{\boldsymbol{Q}}
\newcommand{\Val}{V}
\newcommand{\A}{A}
\newcommand{\hatV}{\widehat{V}}
\newcommand{\tildeV}{\widetilde{V}}
\newcommand{\tV}{\tildeV}
\newcommand{\ts}{\widetilde{s}}
\newcommand{\Q}{Q}
\newcommand{\hatQ}{\widehat{Q}}
\newcommand{\hatA}{\widehat{A}}
\newcommand{\bpi}{\boldsymbol{\pi}}
\newcommand{\bell}{\boldsymbol{\ell}}

\newcommand{\sw}{\mathrm{sw}}

\renewcommand{\P}{\mathbb{P}}

\newcommand{\pluseq}{\mathrel{{+}{=}}}

\renewcommand{\enspace}{\,}

\newcommand{\pa}[1]{\left(#1\right)}

\newcommand{\termA}{\mathbf{(A)}}
\newcommand{\termB}{\mathbf{(B)}}
\newcommand{\termC}{\mathbf{(C)}}
\newcommand{\termD}{\mathbf{(D)}}

\renewcommand{\tilde}{\widetilde}
\renewcommand{\hat}{\widehat}

\usepackage[most]{tcolorbox}
\newtcolorbox{nbox}[1][]{
  enhanced,
  fonttitle=\scshape,
  #1
}

\renewcommand{\bullet}{\!\cdot\!}

\usepackage[round]{natbib}

\begin{document}

\runningtitle{Narrowing the Gap between Adversarial and Stochastic MDPs via Policy Optimization}

\twocolumn[
\aistatstitle{Narrowing the Gap between Adversarial and Stochastic MDPs\\ via Policy Optimization}
\aistatsauthor{ Daniil Tiapkin \And Evgenii Chzhen \And Gilles Stoltz}
\aistatsaddress{ École Polytechnique \\ Université Paris-Saclay \And Université Paris-Saclay {\&} CNRS \And Université Paris-Saclay {\&} CNRS}
]

\begin{abstract}
\vspace{-.25cm}

We consider learning in adversarial Markov decision processes [MDPs] with an oblivious adversary in a full-information setting. The agent interacts with an environment during $T$ episodes, each of which consists of $H$ stages; each episode is evaluated with respect to a reward function revealed only at its end.
We propose an algorithm, called \texttt{APO-MVP}, achieving a regret bound of order $\tilde{\cO}(\mathrm{poly}(H)\sqrt{SAT})$, where $S$ and $A$ are sizes of the state and action spaces, respectively. This result improves upon the best-known regret bound by a factor of $\sqrt{S}$, bridging the gap between adversarial and stochastic MDPs, and matching the minimax lower bound $\Omega(\sqrt{H^3SAT})$ as far as the dependencies in $S,A,T$ are concerned.
The proposed algorithm and analysis avoid the typical tool of occupancy measures, commonly used in adversarial tabular MDPs, and
perform instead policy optimization based only on dynamic programming and on a black-box online linear optimization strategy run over estimated advantage functions, making it easy to implement.
The analysis leverages policy optimization based on online linear optimization strategies
\citep{JMS23} and a refined martingale analysis of the impact on values of estimating transitions kernels
by~\citet{zhang2023settling}.
\end{abstract}

\section{Introduction}
\label{sec:intro}

We study adversarial Markov decision processes [MDPs], introduced by~\citet{even2009online} and~\citet{yu2009markov}, in an episodic setup with full monitoring. Unlike the standard setup, the reward function is not known to the learner beforehand and is revealed sequentially at the end of each episode.

To deal with this problem,
many earlier works relied on online linear optimization [OLO] strategies (see the monograph by~\citealp{CBL06}
for a survey) in the space of so-called occupancy measures \citep{zimin2013online}. These occupancy measures concern the state-action pairs within an episode induced by a given policy and transition kernel. This family of algorithms, known as \texttt{O-REPS}, has been extended to handle unknown transition kernels and bandit feedback by several studies \citep{RoMa19,RM19, JJTLSY20, JHL21}, using an exploration mechanism similar to \texttt{UCRL2} \citep{auer2008near}. However, this type of exploration leads to an additional $\sqrt{S}$ factor in the regret, where $S$ is the number of states, compared to the state of the art in the non-adversarial case \citep{AGM17,dann2017unifying,jin2018q}. Furthermore, \texttt{O-REPS}-based approaches require solving a high-dimensional convex program at each episode, resulting in a non-explicit policy update.

Another line of research has focused on policy-optimization-based approaches for adversarial MDPs, started by works of \citet{cai2020provably} and \citet{shani2020optimistic}. These algorithms use a more practical approach combining dynamic programming with optimization directly in the policy space, instead of working with occupancy measures.
Moreover, this approach is known to be highly practical due to its connection to the well-known \texttt{TRPO} \citep{schulman2015trust} and \texttt{PPO} \citep{schulman2017proximal} algorithms. However, to the best of our knowledge, policy-optimization-based approaches also suffer from an additional $\sqrt{S}$ factor in the regret bound when specialized to finite MDP settings, compared to the state-of-the-art in the stochastic case \citep{AGM17}.

To date, the question of whether dependency on the number of states can be matched between adversarial and stochastic cases remains open. We take the first step towards unifying these rates.

In this work, we use a black-box policy optimization approach, departing from the current state-of-the-art algorithms based on occupancy measures.
This approach of policy optimization based on running online linear optimization strategies in a black-box way on estimated advantage functions
was recently introduced by \citet{JMS23}.
The dynamic programming counterpart of our algorithm, as well as a part of the analysis, relies on the Monotonic Value Propagation [\texttt{MVP}] algorithm of~\citet{zhang2021reinforcement,zhang2023settling}, which allowed to achieve optimal regret bounds up to second-order terms. However, since we do not yet target the lower-order terms, we significantly simplify their approach and provide an arguably more transparent exposition thereof.
All in all, our policy-optimization-based algorithm achieves a $\tilde{\cO}(\mathrm{poly}(H)\sqrt{SAT})$ regret, where
$A$ is the number of actions and $T$ is the number of episodes, and where we recall that $H$ is the length of an episode and $S$ is the number of states.
This result improves on the previous regret bound of \citet{RoMa19} by a factor of $\sqrt{S}$, although it introduces an additional $\mathrm{poly}(H)$ factor. It also matches the minimax lower bound derived for the stochastic case \citep{jin2018q, domingues2021episodic} in all parameters except $H$.

Therefore, we demonstrate that while \emph{policy optimization is already known to be practical, it is also more sample-efficient in large state spaces compared to existing \texttt{O-REPS}-based methods.}

\textbf{Contributions.}
This paper puts forward the following contributions, in the setting of adversarial episodic MDPs with full information: i)
we introduce a algorithm called Adversarial Policy Optimization based on Monotonic Value Propagation
(\texttt{APO-MVP}) that relies on a black-box online linear optimization solver and on dynamic programming, making it easier to implement in practice; ii) we demonstrate that the proposed algorithm is able to achieve a $\tcO(\mathrm{poly}(H)\sqrt{TSA})$ regret, improving on the previously best-known dependency on the number of states $S$ and achieving the minimax lower bound $\Omega(\sqrt{H^3SAT})$ in all parameters, except $H$; iii) our analysis is modular and rather general, providing high flexibility and providing new tools for the study of adversarial MDPs with policy optimization.

More details on the reasons for the success in shaving off the $\sqrt{S}$ factor present in earlier analysis in this
context may be found at the beginning of Section~\ref{sec:decomp}.

\textbf{Notation.} For any positive integer $N$, we denote by $[N] \eqdef \{1,\ldots,N\}$ and $[N]^* \eqdef \{0\} \cup [N]$ the
sets of the first positive and non-negative integers not greater than $N$, respectively.
For $a, b \in \R$, we denote by $a \vee b$ and $a \wedge b$ the maximum and the minimum between $a$ and $b$, respectively.
For a finite set $\cE$, we denote by $\Delta(\cE)$ the set of probability distributions over $\cE$.
We refer to natural logarithms by $\ln$ and to logarithms in base~2 by $\log_2$.
When we write $\tilde{\cO}(\,\cdot\,)$, we hide all absolute constants and polylog multiplicative terms. We also adopt an agreement that $\max\{\emptyset\} = 0$.

\subsection{Related work}

\textbf{Adversarial MDPs.}
The concept of adversarial (or online) MDPs was introduced by seminal works of \citet{even2009online} and \citet{yu2009markov}. In the case of MDPs with known transition kernel, \citet{neu2010online} and \citet{zimin2013online} provided solutions, but dealing with the case of unknown transition kernels remained open till the work of \citet{RoMa19}, who extended the algorithm of \citet{zimin2013online} to this setting. Subsequent developments considered the extension to bandit feedback \citep{JJTLSY20}, linear mixture MDPs \citep{cai2020provably,he2022near}, delayed feedback \citep{lancewicki2022learning}, and linear MDPs \citep{luo2021policy}.  Despite these results, the setting of unknown transition kernel in finite MDPs with full-information feedback was not yet fully understood. Specifically, the upper bound established by \citet{RoMa19} does not match the best-known lower bound, leaving a gap in the theoretical understanding of this problem.

\textbf{Policy optimization in adversarial MDPs.}
Policy optimization methods for adversarial MDPs with unknown transitions were first studied by \citet{cai2020provably} and \citet{shani2020optimistic}. The first work focused on the case of linear mixture MDPs with full-information feedback, later improved by \citet{he2022near}, while the second one considered the case of tabular MDPs with bandit feedback, with \citet{luo2021policy} refining the approach to achieve a state-of-the-art regret bound. More recent work has primarily addressed policy optimization under function approximation settings, including linear mixture MDPs \citep{he2022near} and linear MDPs \citep{sherman2023rate,liu2024towards,cassel2024warmup}. However, none of the methods was able to achieve improvements upon the occupancy-measure-based methods in the case of tabular MDPs in terms of a regret bound.

\section{Problem Formulation}
\label{sec:setup}

An $H$--episodic (obliviously) adversarial Markov decision process (MDP),
where $H \geq 1$, is determined by
a finite set of states $\cS$, with cardinality $S$,
a finite set of actions $\cA$, with cardinality $A$,
a sequence $\bP = (P_h)_{h \in [H-1]}$ of Markov transition kernels $P_h\colon \cS \times \cA \to \Delta(\cS)$,
and by a (potentially adversarially chosen) fixed-in-advance sequence $(\br_t)_{t \geq 1}$
of bounded time-inhomogeneous $H$--episodic reward functions. Each reward function
is of the form $\br_t = (r_{t, h})_{h\in[H]}$, where $r_{t,h} \colon \cS \times \cA \to [0, 1]$.
For simplicity (and with no loss of generality, up to resorting to some doubling trick),
we assume that the number $T$ of episodes is fixed and known.
We set some initial state $s_1$ for each episode.
At each episode $t$ and at each stage $h$, based on past observations, the learner picks
a stage policy $\pi_{t, h} \colon \cS \to \Delta(\cA)$ to draw the action.
The interaction with the environment is therefore governed by the following protocol.
For each episode $t = 1, \ldots, T$:
\begin{enumerate}[noitemsep,topsep=-3pt,partopsep=-3pt]
    \item Reset state $s_{t, 1} = s_1$;
    \item Start new episode --- for each stage $h = 1, \ldots, H$:
    \begin{itemize}[topsep=0pt, itemsep=0pt, leftmargin=15pt]
        \item Pick a policy $\pi_{t, h}$ and
        sample an action $a_{t, h} \sim \pi_{t, h}(\,\cdot \mid s_{t, h})$;
        \item If $h \leq H-1$, move to the next state $s_{t, h+1} \sim P_h( \,\cdot \mid s_{t, h}, a_{t, h})$;
    \end{itemize}
    \item Observe the reward function $\br_{t} = (r_{t, h})_{h \in [H]}$.
\end{enumerate}
We compare the performance of the policies $\bpi_t = (\pi_{t,h})_{h \in [H]}$ picked
to the one achieved by resorting to a static policy $\bpi = (\pi_{h})_{h \in [H]}$ in each episode,
in terms of value functions.
We define the value function of a policy $\bpi$, at episode $t \in [T]$, and started from step $h \in [H]$, as
\begin{align*}
\Val^{\bpi,\br_t,\bP}_{h}(s) \eqdef \E_{\bpi,\bP}\!\left[\sum_{j = h}^H r_{t,j}(s_{t, j}, a_{t, j}) \biggm| s_{t, h} = s\right];
\end{align*}
we recall the environment (reward functions, transition kernels) in the notation for value functions and expectations,
as environments will vary in the algorithm and analysis.
The regret of the learner is defined as the difference between the accumulated value of the best static policy in
hindsight and the gained value of the learner, that is,
\begin{equation}
\label{eq:defRT}
R_T \eqdef \max_{\bpi} \sum_{t = 1}^T \pa{\Val^{\bpi,\br_t,\bP}_{1}(s_1) - \Val^{\bpi_t,\br_t,\bP}_{1}(s_1)}.
\end{equation}
The goal of the learner is to design policies $(\bpi_t)_{t \in [T]}$ minimizing the above-defined regret.

\paragraph{Additional notation.}
For the analysis, we define $\Q$--value functions and remind Bellman's equations.
For any policy $\bpi$, we define the $\Q$--value function at episode $t \in [T]$, and started from step $h \in [H]$, as
\begin{multline*}
\Q_{h}^{\bpi,\br_t,\bP}(s, a) \\
\eqdef \E_{\bpi,\bP}\!\left[\sum_{j = h}^H r_{t,j}(s_{t, j}, a_{t, j}) \biggm| s_{t, h} = s, \ a_{t, h} = a\right].
\end{multline*}

The advantage function is in turn defined as $\A_{h}^{\bpi,\br_t,\bP}(s, a) \eqdef \Q_{h}^{\bpi,\br_t,\bP}(s, a)
- \Val^{\bpi,\br_t,\bP}_{h}(s)$.

We use the usual convention that for a transition kernel $K \colon \cS \times \cA \to \Delta(\cS)$,
a policy $\pi \colon \cS \to \Delta(\cA)$, and two functions $f \colon \cS \to \R$ and $g \colon \cS \times \cA \to \R$,
\begin{align*}
K \bullet f (s, a) & \eqdef \sum_{s' \in \cS} K(s' \mid s,a) \, f(s') \\
\text{and}\qquad\qquad
\pi \bullet g(s) & \eqdef \sum_{a \in \cA} \pi(a \mid s) \, g(s,a)\enspace.
\end{align*}
Then, Bellman's equations read for all episodes $t \in [T]$, steps $h \in [H-1]$, and policies $\bpi$, as
\begin{align*}
\Q_{h}^{\bpi,\br_t,\bP}(s, a) & = r_{t, h}(s, a) +  P_h \bullet \Val_{h+1}^{\bpi,\br_t,\bP}(s, a) \\
\text{and}\qquad
\Val_{h}^{\bpi,\br_t,\bP}(s) & = \pi_{t, h} \bullet \Q_{h}^{\bpi,\br_t,\bP}(s)\enspace,
\end{align*}
while for $h = H$, one has $\Q_{H}^{\bpi,\br_t,\bP}(s, a) = r_{t, H}(s, a)$ as well as
$\Val_{H}^{\bpi,\br_t,\bP}(s) = \pi_{t, H} \bullet \Q_{H}^{\bpi,\br_t,\bP}(s)$.

\section{Algorithm and Main Result}
\label{sec:algorithm}

In this section, we first describe our algorithm,
\texttt{APO-MVP}, which stands for Adversarial Policy Optimization
based on Monotonic Value Propagation,
and then state the performance bound obtained.

\subsection{Algorithm \texttt{APO-MVP}}
\label{sec:defalgo}

\newcommand{\algoshorthere}{
\begin{algorithm}[tb]
    \KwData{Confidence level $\delta$; numbers $T$, $S$, $A$, $H$; \\
    online linear optimization strategy $\varphi = (\varphi_t)_{t \geq 1}$ \smallskip}
    Define $J = 2SATH\log_2(2T)/ \delta$; \\
    Initialize $\hatP_h \equiv 1/S$, $n_h \equiv 0, \cH_{h,s} \equiv \emptyset$ for all $(h,s)$\;
    Initialize $\pi_{1, h}(\,\cdot \mid s) = \phi_1(\emptyset)$ for all $(h,s)$\;
    Set initial state $s_1 \in \cS$\;
    \For{rounds $t = 1,\ldots, T$}{
        Set $s_{t,1} = s_1$\;
        \For{$h=1,\ldots,H$}{
            Play action $a_{t,h} \sim \pi_{t,h}(\,\cdot \mid s_{t,h})$\;
            Receive $s_{t,h+1} \sim P_h(\,\cdot \mid s_{t,h}, a_{t,h})$\;
            Update counters $n_h(s_{t,h}, a_{t,h}) \pluseq 1 $ \\
            ~ \qquad\quad and $n_h(s_{t,h}, a_{t,h}, s_{t,h+1}) \pluseq 1$\;
            \If{$\exists \, \ell \geq 1: n_h(s_{t,h}, a_{t,h}) = 2^{\ell-1}$}{
                \!\!$\displaystyle{\forall s', \ \ \hatP_h(s' | s_{t,h}, a_{t,h}) = \frac{n_h(s_{t,h}, a_{t,h}, s')}{2^{\ell-1}}}$\;
                \!\! $\displaystyle{b_{h}(s_{t,h}, a_{t,h}) = \sqrt{\frac{2H^2\ln(J)}{2^{\ell - 1}}} \wedge H}$\;
                Activate \texttt{trigger}\;
            }
        }
        Receive a reward function $\br_t = (r_{t, h})_{h \in [H]}$\;
        \eIf{\texttt{trigger}}{
            Set $\cH_{h,s} = \emptyset$ for all $(h,s)$\;
            Deactivate \texttt{trigger}\;
        }{
        Compute $\hat{Q}_{t,h}$, $\hat{V}_{t,h}$, and $\hat{A}_{t,h}$ via \eqref{eq:value_update}\;
        Add $\hat{A}_{t,h}(s,\cdot)$ to $\cH_{h,s}$ for all $(h,s)$\;
        }
        Let $\pi_{t+1,h}(\,\cdot \mid s) = \phi_t(\cH_{h,s})$ for all $(h,s)$ \;
    }
    \caption{\label{alg:main_short} Adversarial Policy Optimization based on Monotonic Value Propagation (\texttt{APO-MVP})}
\end{algorithm}
}

Let us start with a high-level description of the proposed algorithm, and details will be provided below.
Similarly to~\citet{RoMa19}, our algorithm proceeds in random epochs $\cE_e \subseteq [T]$
indexed by $e = 1, 2, \ldots, m(T)$ of random lengths denoted by $E_1, \ldots, E_{m(T)} \in [T]$, i.e.,~$E_e \eqdef |\cE_e|$.
At the beginning of each epoch $e$,
\[
\mbox{estimates} \ \ \hatbP^{(e)} = \Bigl( \hatP^{(e)}_h \Bigr)_{h \in [H-1]}
\]
and bonus functions $\bb^{(e)} = \bigl( b_{h}^{(e)} \bigr)_{h \in [H]}$,
where $b_{h}^{(e)} \colon \cS \times \cA \to [0,H]$, are computed and will be used during the entire epoch~$e$,
as detailed in~\eqref{eq:defPe}--\eqref{eq:defbe} and in Fact~\ref{fact:e}.
Actually, $b_{H}^{(e)}$ will be identically null, but we consider it so that
the bonus functions $\bb^{(e)}$ may be added to reward functions $\br_t$.

\paragraph{Within-epoch statement.}
We now explain the updates and choices made at each episode $t \in [T]$.
First, the policies $\bpi_t$ are picked, as indicated below.
Then, denoting by $e_t$ the epoch such that $t \in \cE_{e_t}$,
at the end of episode $t$, i.e., once $\br_t$ is revealed,
we build optimistic estimates of the $\Q$--value and value functions in a backward fashion,
based on Bellman's equations:
$\hatQ_{t,H}(s,a) \eqdef r_{t,H}(s,a)$
 and $\hatV_{t,H}(s) \eqdef \pi_{t,H} \bullet \hatQ_{t,H}(s)$
and for $h \in [H-1]$,
\begin{align}
    \hatQ_{t,h}(s,a) &\eqdef r_{t,h}(s,a) + b_{h}^{(e_t)}(s,a) +  \hatP_{h}^{(e_t)} \bullet \hatV_{t,h+1}(s,a)\,,\notag \\
    \hatV_{t,h}(s) &\eqdef \pi_{t,h} \bullet \hatQ_{t,h}(s)\,.\label{eq:value_update}
\end{align}
For all $h \in [H]$, estimated advantage functions are defined by $\hatA_{t,h}(s,a) \eqdef \hatQ_{t,h}(s,a) - \hatV_{t,h}(s)$,
and we denote $\hatA_{t,h}(s,\,\cdot\,) \eqdef \bigl( \hatA_{t,h}(s,a) \bigr)_{a \in \cA}$.

The policies $\bpi_t = (\pi_{t,h})_{h \in [H]}$ are picked based on an online linear optimization [OLO]
strategy $\varphi = (\varphi_t)_{t \geq 1}$, which is a sequence of functions $\varphi_t \colon (\R^\cA)^{t-1} \to \Delta(\cA)$
satisfying some performance guarantee stated in Definition~\ref{def:adv_main}. (The function $\varphi_1$ is constant.)
We run $S H$ such strategies in parallel as follows:
$\forall (s, h) \in \cS \times [H]$,
\begin{equation}
\label{eq:policies}
\pi_{t,h}(\,\cdot \mid s) = \varphi_t \biggl( \Bigl( \hatA_{\tau,h}(s,\,\cdot\,)\Bigr)_{\tau \in \cE_{e_t} \cap [t-1]} \biggr)\,.
\end{equation}
Note that these choices indeed exploit information available at the beginning of episode~$t$ (at the end of
episode $t-1$), and rely only on the estimated advantage functions of the current epoch.
One may see $\phi$ as an adaptive version of \texttt{PPO}- or \texttt{TRPO}-like updates~\citep{schulman2015trust,schulman2017proximal}.
We will consider, for the sake of concreteness, the polynomial-potential- and exponential-potential-based strategies (see Examples~\ref{ex:polpot} and~\ref{ex:adahedge} and references therein), but many other OLO strategies
would work.
Appendix~\ref{sec:algorithm_app} states closed-form expressions
of the policies constructed with these strategies as well as their computational complexities.

\begin{remark}[Two technical remarks]
\label{rk:clipping}
The kernel estimate and the bonus functions are fixed within a given epoch, which is the main reason why we are
able to provide a black-box treatment of the problem relying on any OLO strategy satisfying Definition~\ref{def:adv_main}.

As the reward function takes values in $[0,1]$, the $\Q$--value functions are bounded by $H$, and it is a common practice in the case of
non-adversarial reward functions to clip the estimates to $[0,H]$ (see, e.g., \citealp{AGM17}), which only helps.
Unfortunately, our adversarial analysis related to the OLO part of the proof heavily relies on the so-called performance-difference lemma~\citep{kakade2002approximately}, which does not hold once clipping is involved. Thus, we opt out from clipping, paying an additional $H$ factor at the eventual regret bound of Theorem~\ref{thm:main_result} but still improving the dependency on~$S$.
Successful incorporation of clipping could improve the regret by an $H$ multiplicative factor. A possible alternative to the clipping mechanism, connected to a notion of contracted MDPs, was recently proposed by \citet{cassel2024warmup} and might form a viable option to improve the dependency of our regret bound in $H$.
\end{remark}

\textbf{Epoch switching.}
The epoch-switching conditions below were also considered and analyzed by \citet{zhang2023settling}.
We introduce the following empirical counts, for all episodes $t \in [T]$, stages $h \in [H-1]$, state--action pairs
$(s,a) \in \cS \times \cA$, and states $s' \in \cS$:
\begin{align*}
n_{t, h}(s, a, s') &\eqdef \sum_{\tau = 1}^t \ind\big\{(s_{\tau, h,}, a_{\tau, h}, s_{\tau,h+1}) = (s,a,s')\big\}\,,\\
n_{t, h}(s, a) &\eqdef \sum_{s' \in \cS} n_{t, h}(s, a, s')\,.
\end{align*}
We start at epoch $e=1$. When for some $(t, h, s, a) \in [T] \times [H-1] \times \cS \times \cA$, the count $n_{t, h}(s, a)$ equals
$2^{\ell - 1}$ for some integer $\ell \geq 1$, the next epoch is started at episode $t+1$.

Now, for each episode $t \in [T]$, stage $h \in [H-1]$, and state--action pair $(s,a) \in \cS \times \cA$, we denote by
\[
\sw_{t,h}(s,a) \eqdef
\max\bigl\{ \tau \in [t] : \exists \ell \geq1,\, n_{\tau,h}(s,a) = 2^{\ell-1} \bigr\}\,.
\]
In words, $\sw_{t,h}(s,a)$ is the last episode when an epoch switch took place because, among others, of $(s,a)$. We
refer to the largest value of $\ell$ in the maximum defining $\sw_{t,h}(s,a)$ by
\[
\ell_{t,h}(s,a) \eqdef
\max\bigl\{ \ell \geq 1 : n_{t,h}(s,a) \geq 2^{\ell-1} \bigr\}\,.
\]

The values $\ell_{t,h}(s,a)$ index local epochs for a given state--action pair $(s,a)$,
while the global epochs $e_t$ are defined based on all local epochs.
(More details may be found in Section~\ref{sec:add-concepts}.)

\begin{fact}
\label{fact:local-global}
By design, the functions $\sw_{t-1,h}$ and $\ell_{t-1,h}$ defined above (note the
subscripts $t-1$ here) are identical
for all episodes $t \in \cE_e$ of a given epoch~$e$.
\end{fact}

We may now define the estimated transition kernels $\hatbP_t$ and bonus functions $\bb_t$:
for all $(s,a) \in \cS \times \cA$ and $s' \in \cS$,
first for all $h \in [H-1]$,
\begin{align}
\label{eq:defPe}
\hatP_{t,h}(s' \mid s,a) {\eqdef}
\begin{cases}
1/S & \mbox{if} \ n_{\tau,h}(s,a) {=} 0\,, \\[.1cm]
\displaystyle{\frac{n_{\tau,h}(s,a,s')}{n_{\tau,h}(s,a)}} & \mbox{if} \ n_{\tau,h}(s,a) {\geq} 1\,, \\
\end{cases}
\end{align}
with $\tau = \sw_{t-1,h}(s,a)$, and
\begin{align}
\label{eq:defbe}
b_{t,h}(s, a) & \eqdef
\begin{cases}
H & \mbox{if} \ \ell = 0\,, \\[.1cm]
\displaystyle{\sqrt{\frac{2 H^2 \ln(J)}{ 2^{\ell-1} }} \wedge H}
& \mbox{if} \ \ell \geq 1\,, \\
\end{cases}
\end{align}
with $J = 2SATH\log_2(2T)/ \delta$ and $\ell = \ell_{t-1,h}(s,a)$.
We also set, by convention, $b_{t,H}(s, a) = 0$. In particular, $\hatP_{t,h}(\,\cdot \mid s,a)$
corresponds to an empirical frequency vector based on $2^{\ell_{t-1,h}(s,a)-1}$
values when $\ell_{t-1,h}(s,a) \geq 1$.

\begin{fact}
\label{fact:e}
By Fact~\ref{fact:local-global},
the above-defined $\hatbP_t$ and $\bb_{t}$ are indeed identical over
all episodes $t \in \cE_e$ of a given epoch~$e$.
\end{fact}

\textbf{Summary.}
The strategy outlined above is summarized in the algorithm sketch titled Algorithm~\ref{alg:main_short}; for particular choices of function $\phi$, see Examples~\ref{ex:polpot} and~\ref{ex:adahedge}. We refer to Appendix~\ref{sec:algorithm_app} for a more detailed algorithmic
description, titled~\Cref{alg:main}, including the details of the computation of $\hat{Q}_{t,h}$, $\hat{V}_{t,h}$, and $\hat{A}_{t,h}$.
\algoshorthere

It is important to note that the computational complexity of the algorithm is equivalent to that of standard dynamic programming methods, as the policy optimization phase introduces no additional computational overhead beyond the requirements of transition kernel estimation and value function updates.

\subsection{Main Result}

We may now state our main result and discuss its relation to previously known bounds.

\begin{theorem}[Main theorem]
\label{thm:main_result}
Algorithm \texttt{APO-MVP}, used, for instance, with the OLO
strategies based on polynomial or exponential potential (see Examples~\ref{ex:polpot} and~\ref{ex:adahedge}),
satisfies, with probability at least $1-3\delta$,
\begin{align*}
R_T \leq & \phantom{+} \sqrt{H^7SA \, T \log_2(2T)} \, \bigl( 2\log_2(2T) + 16\sqrt{\ln(A)} \bigr) \\
& + 7\sqrt{H^4 SA \, T\ln\bigl(2SATH\log_2(2T)/\delta \bigr)}\\
&+ 2\sqrt{2 H^6 \, T \log_2(2T) \ln(2/\delta)} + 2 H^3 SA \,.
\end{align*}
\end{theorem}

\begin{proof}
The result follows the decomposition stated in the introduction of Section~\ref{sec:decomp}
together with Lemmas~\ref{lem:termB_main}--\ref{lem:optimism}--\ref{lem:termD_main}--\ref{lem:termC_main}
located therein.
\end{proof}

Theorem~\ref{thm:main_result} shows that the regret is $\tilde{\cO}(\sqrt{H^7SAT})$, matching the minimax lower bound $\Omega(\sqrt{H^3 SAT})$ for stochastic MDPs in terms of dependencies on $S$, $A$, and $T$, up to logarithmic factors \citep{jin2018q,domingues2021episodic}. To the best of our knowledge, it is the first result that achieves the minimax optimal dependency on the number of states $S$ in the adversarial setting.

\paragraph{Comparison to Rosenberg {\&} Mansour (2019b).}
Algorithm \texttt{UC-O-REPS} by \citet{RoMa19} achieves $\tilde{\cO}(\sqrt{H^4S^2AT})$ regret bound in our setting and with our notation
(taking $L=H$ and $|\cX| = HS$ since a state-space layer $\cX$ may be represented as $H$ independent copies of $\cS$). In particular, our result improves upon the previous best known bound in the regime of large state spaces $S \geq H^3$.
We suspect that---perhaps through successful incorporation of clipping, see Remark~\ref{rk:clipping}---the
regret bound could be improved to $\tilde{\cO}(\sqrt{H^5SAT})$. We plan to investigate this in future works,
and it remains an open problem to fully match the minimax lower bound $\Omega(\sqrt{H^3SAT})$.

Finally, our analysis in Section~\ref{subsec:termA} relies significantly on the fact that the adversary is oblivious, while \texttt{UC-O-REPS} can handle fully adversarial setups. However, due to the exploration mechanism used, this algorithm is not able to take advantage of the oblivious adversary and would still pay the same $\sqrt{S}$ factor.

\paragraph{Comparison to \citet{cai2020provably}.} Our algorithm shares similarities with the online proximal policy optimization [\texttt{OPPO}] approach of \citet{cai2020provably} and~\citet{shani2020optimistic}, which also uses dynamic programming and policy optimization through online mirror descent.
However, our approach incorporates the doubling trick to stabilize value updates, enabling us to: i) employ any online linear optimization strategy
in a black-box manner without unnecessary adaptations; and ii) improve the dependency on $S$ by a multiplicative
factor $\sqrt{S}$, by leveraging the analysis of~\citet{zhang2023settling}.

\section{Proof Sketch for Theorem~\ref{thm:main_result}}
\label{sec:decomp}

We decompose the regret into four terms to be treated separately.
Denoting by $\bpi^\star$ as the policy that achieves the maximum in~\eqref{eq:defRT},
we decompose the regret $R_T$, following ideas of \citet{auer2008near} and \citet{AGM17}, as: $R_T = \termA + \termB + \termC + \termD$, where
\begin{allowdisplaybreaks}
\begin{align*}
    \termA
    &\eqdef \sum_{t=1}^T \left( V^{\bpi^\star, \br_t, \bP}_1(s_1) - V^{\bpi^\star, \br_t + \bb_{t}, \hatbP_t}_1(s_1)\right),\\
    \termB &\eqdef \sum_{t=1}^T \left(V^{\bpi^\star, \br_t + \bb_t, \hatbP_t}_1(s_1) - V^{\bpi_t, \br_t + \bb_t, \hatbP_t}_1(s_1)\right),\\
    \termC &\eqdef \sum_{t=1}^T \left( V^{\bpi_t, \br_t + \bb_t, \hatbP_t}_1(s_1) - V^{\bpi_t, \br_t + \bb_t , \bP}_1(s_1)\right),\\
    \termD&\eqdef\sum_{t=1}^T V^{\bpi_t, \bb_t, \bP}_1(s_1)\,,
\end{align*}
\end{allowdisplaybreaks}
~ \!\!\!\!\!\!\!\!\!\!\!\!\! where we used the linearity of the value functions: $\Val^{\bpi, \bg + \bg', \bQ}_1 \equiv \Val^{\bpi, \bg, \bQ}_1 + \Val^{\bpi, \bg', \bQ}_1$.

\textbf{The keys to success / Challenges overcome.}
We now provide a high-level overview of the techniques used for each term.
We do not claim novelty in the very control of any of the terms involved;
novelty lies in the decomposition above of the regret, which somehow separates adversarial and stochastic parts. We then apply or adapt the right tools, some of them being developed recently.
One could argue that prior works on adversarial tabular MDPs have struggled to eliminate the $\sqrt{S}$ factor
precisely because they adopted an overly adversarial approach to the problem. In contrast, our approach stems from deliberately blending adversarial and stochastic elements---a balance that enables us to derive the bound.

The treatment of $\termB$ definitely belongs to the adversarial parts of the proof scheme.
We crucially use that within each epoch the considered transition kernels are constant (see Fact~\ref{fact:e}), so that we may resort to the adversarial-learning techniques reviewed by~\citet{JMS23,JMS25}, which consists of running $SH$
independent OLO strategies on advantage functions or on $\Q$--value functions.
(Advantage functions are preferred in practice, though.) The OLO strategies that are the most popular
in the literature are based on exponential weights (sometimes called multiplicative weights; see
the detailed discussion in \citealp[Section~1.1]{JMS25}), but other strategies are suitable.
We deal with the underlying OLO strategies as black boxes, to get a modular and transparent proof.

Dealing with $\termC$ forms a stochastic part of the proof scheme,
and is actually the most involved part of the analysis from the probabilistic standpoint: we resort to the machinery developed by~\citet{zhang2023settling}, which relies greatly on a doubling trick that we mimicked in the definition of the \texttt{APO-MVP} algorithm.
We however make the argument of~\citet{zhang2023settling} more modular and extract the essential elements of the proof.

Finally, for $\termA$ we leverage the careful choice of the bonuses $\bb_t$ to show that on a properly chosen high-probability event, $\termA$ is non-positive,
while $\termD$ is the least involved term: it can be controlled by some lines of elementary calculations.

In what follows, each section provides additional details of the analysis per term, in the order: $\termB$ -- additional technical concepts -- $\termA$ -- $\termD$ -- $\termC$.

\subsection{Term $\termB$: OLO Analysis}
\label{subsec:termB}

The goal of this section is to prove the following result, which also holds
for other OLO strategies satisfying the performance guarantee of Definition~\ref{def:adv_main}.

\begin{lemma}
    \label{lem:termB_main}
    Among others, the OLO strategies based on polynomial or exponential potentials
    (see Examples~\ref{ex:polpot} and~\ref{ex:adahedge}) satisfy
    \begin{align*}
    \termB \leq 16\sqrt{H^7SAT \log_2(2T)\ln(A)}\,.
    \end{align*}
\end{lemma}

Before proving this result, let us briefly recall what online linear optimization [OLO] consists of;
see the monograph by \citet{CBL06} for a more detailed exposition.
We take some generic notation for now but will later connect OLO
to constructions of policies; in particular, we consider for now reward vectors of length $K \geq 2$,
but will later replace $[K]$ by the action space~$\cA$.

\textbf{Online linear optimization.}
At each round $t \geq 1$ and based on the past, a learning strategy $\varphi = (\varphi_t)_{t \geq 1}$ picks
a convex combination $\bw_t = (w_{t,1},\ldots,w_{t,K}) \in \Delta\bigl([K]\bigr)$ while an opponent
player picks, possibly at random, a vector $\bg_t = (g_{t,1},\ldots,g_{t,K})$ of signed
rewards. Both $\bw_t$ and $\bg_t$ are revealed at the end of the round.
By ``based on the past'', we mean, for the learning strategy, that
$\bw_t = \varphi_t\bigl( (\bg_\tau)_{\tau \leq t-1} \bigr)$.
The initial vector $\bw_1$ is constant.

\begin{definition}
\label{def:adv_main}
A learning strategy $\varphi$ controls the regret in the adversarial setting
with rewards bounded by $M > 0$
if there exists a sequence $(B_{T,K})_{T \geq 1}$ with $B_{T,K}/T \to 0$
such that, against all opponent players sequentially picking reward vectors $\bg_t$
in $[-M,M]^K$, for all $T \geq 1$,
\[
\max_{k \in [K]} \sum_{t=1}^T g_{t,k} - \sum_{t=1}^T \sum_{j \in [K]} w_{t,j}\,g_{t,j}
\leq 2M\,B_{T,K}\,.
\]
\end{definition}

The optimal orders of magnitude of $B_{T,K}$ are $\sqrt{T \ln(K)}$.
In Definition~\ref{def:adv_main}, the strategy may know $M$ and rely on its value.
Also, the strategy should work for any optimization horizon~$T$ (see the final ``for all $T \geq 1$''
in the definition above): this is because
the lengths $E_e$ of the global epochs $\cE_e$ are not known in advance.
There exist several strategies meeting the requirements of Definition~\ref{def:adv_main};
we provide two examples.

\begin{example}
\label{ex:polpot}
The potential-based strategies by \citet{CBL03}
are defined based on a non-decreasing function $\Phi \colon \R \to [0,+\infty)$.
They resort to $w_{1,k} = \tfrac{1}K$ and for $t \geq 2$, to $w_{t,k} = \tfrac{v_{t,k}}{{\sum_{j \in [K]} v_{t,j}}}$, where
\begin{equation}
\label{eq:def-pot-based}
v_{t,k} = \Phi\!\left( \sum_{\tau=1}^{t-1} g_{\tau,k} - \sum_{\tau=1}^{t-1} \sum_{j \in [K]} w_{\tau,j} g_{\tau,j} \right).
\end{equation}
For $\Phi \colon x \mapsto \bigl( \max\{x,0\} \bigr)^{2 \ln(K)}$, the polynomial potential,
\citet[Section~2]{CBL03} show that
the strategy satisfies the performance guarantee of Definition~\ref{def:adv_main}
with $B_{T,K} = \sqrt{6 T \ln(K)}$.
\end{example}

\begin{example}
\label{ex:adahedge}
\citet{ACBG02} studied the use of exponential potential
with time-varying learning rates $\eta_t = (1/M) \sqrt{\ln(K)/t}$,
i.e., using $\Phi_t(x) = \exp(\eta_t x)$ in~\eqref{eq:def-pot-based}
to define the weights at round $t$.
This strategy satisfies the performance guarantee of Definition~\ref{def:adv_main}
with $B_{T,K} = \sqrt{T \ln(K)}$.
\end{example}

There exist adaptive versions of the previous two strategies:
\texttt{ML-Poly} in \citet{MLPoly}, \texttt{AdaHedge} in \citet{EKRG11}, \citet{JMLR:v15:rooij14a}, \citet{OP15}.

Appendix~\ref{sec:algorithm_app} states closed-form expressions
of the policies~\eqref{eq:policies} constructed with the strategies of
Examples~\ref{ex:polpot} and~\ref{ex:adahedge}, as well as \texttt{AdaHedge}.

\textbf{Connection between OLO and the construction of policies.}
\citet{JMS23,JMS25} prove the following.
Let $r'_{t,h}\colon \cS \times \cA \to [0,M]$ be a sequence of reward functions.
Define a sequence of policies $(\bpi'_t)_{t \geq 1}$ as:
for each $t \geq 1$, for each $s \in \cS$, for each $h \in [H]$,
\begin{align*}
    &\pi'_{t,h}(\,\cdot\,|s) = \varphi_t \biggl( \Bigl( A^{\bpi'_\tau,\br'_\tau,\bP'}_{h}(s,\,\cdot\,) \Bigr)_{\tau \leq t-1} \biggr)\\
    &\mbox{where} \ \ \ A^{\bpi'_\tau,\br'_\tau,\bP'}_{h}(s,\,\cdot\,) = \Bigl( A^{\bpi'_\tau,\br'_\tau,\bP'}_{h}(s,a) \Bigr)_{a \in \cA}
\end{align*}

\begin{restatable}[\citealp{JMS23,JMS25}]{lemma}{lmaggreg}
\label{lem:OLO_main}
If the learning strategy satisfies the conditions of Definition~\ref{def:adv_main},
then the sequence of policies defined right above is such that,
for all fixed policies $\bpi = (\pi_h)_{h \in [H]}$,
for all $T \geq 1$,
\[
\sum_{t=1}^T \Bigl( V_1^{\bpi,\br'_t,\bP'}(s_1) - V_1^{\bpi'_t,\br'_t,\bP'}(s_1) \Bigr) \leq 2MH^2 \, B_{T,A} \, .
\]
\end{restatable}
For completeness,
the proof of Lemma~\ref{lem:OLO_main} is provided in Appendix~\ref{app:B}.
We are ready to prove Lemma~\ref{lem:termB_main}.

\noindent \emph{Proof of Lemma~\ref{lem:termB_main}.}
We apply Lemma~\ref{lem:OLO_main} in each global epoch~$\cE_e$,
with $\bP' = \hatbP_{t}$ (see Fact~\ref{fact:e})
and $\br'_t = \br_t + \bb_t$ for all $t \in \cE_e$.
Since $r_{t,h} \in [0, 1]$ and $b_{t, h} \in [0, H]$, we can pick $M = 1 + H \leq 2H$.
Decomposing term $\termB$ into a summation over the global epochs and using the bound of Lemma~\ref{lem:OLO_main}
for each of them, we deduce that, for both strategies of Examples~\ref{ex:polpot} and~\ref{ex:adahedge},
\begin{equation}
\label{eq:termB_interm}
\begin{aligned}
        \termB \leq 4H^3\sum_{e = 1}^{m(T)} B_{E_e, A}
        &\leq
        4 H^3 \sum_{e = 1}^{m(T)} \sqrt{6 E_e \ln(A)}\\
        &\leq
        16 H^3 \sqrt{T m(T) \ln(A)}\,,
\end{aligned}
\end{equation}
where we applied Jensen's inequality to the root.
Lemma~\ref{lem:number_of_epochs} below then yields the claimed result.\qed

\subsection{Additional Technical Concepts}
\label{sec:add-concepts}

To deal with the remaining terms $\termA$--$\termD$--$\termC$,
we will not need anymore to pay attention to global epochs $\cE_{e_t}$,
only local epochs $\ell_{t,h}(s,a)$
are of interest.

We review two concepts which have been successfully used by~\citet{zhang2023settling}
to derive minimax optimal regret bounds in the case of stochastic MDPs.

\textbf{The first concept: epoch-switching conditions and profiles.}
The functions indicating local epochs $\ell_{t,h}(s,a)$ were called a profile by~\citet{zhang2023settling};
they take bounded values: $\ell_{t, h} \colon \cS \times \cA \to \bigl[ \ilog{2}{T} \bigr]^*$,
and let
\begin{equation}
\label{eq:profiles}
\bell_t = (\ell_{t, h})_{h \in [H-1]}
\end{equation}
with the agreement $\ell_{0,h}(\,\cdot\,, \,\cdot\,) \equiv 0$ for $h \in [H-1]$.
We also introduce $\bell_{< t} = (\bell_\tau)_{0 \leq \tau \leq t-1}$.
Using the above-defined profiles, we note that the global epoch $e_t$ of a given episode $t \in [T]$
may be obtained as a function of $\bell_{< t}$, namely,
\begin{align}
\label{eq:epoch_from_profile}
e_t = {\sum_{\tau = 1}^{t-1}\bigg(\sum_{(s{,}a{,}h)} \bigl(
\ell_{\tau, h}(s,a) {-} \ell_{\tau {-} 1, h}(s,a) \bigr)\bigg) \wedge 1} .
\end{align}
Indeed, if the counter of no triplet $(s, a, h)$ has reached a value of the form $2^r$ for some integer $r$
by passing from episode $\tau-1$ to $\tau$, then the summation in the minimum is zero,
meaning that the episodes $\tau$ and $\tau+1$ belong to the same global epoch.
On the contrary, if the counter of at least one $(s, a, h)$ reached such a value, then this sum is at least~$1$ (there can be more than one triplets
satisfying this), meaning that $\tau$ and $\tau+1$ belong to different global epochs. Thus, thanks to the minimum, the above quantity counts the number of
(global) epoch switches from $\tau=1$ to $\tau=t$.
In other words, the global epoch $e_t$ is uniquely determined by the preceding profiles.

Since there are $S A (H-1)$ different triplets $(s, a, h)$ and each such triplet is associated with at most $\ilog{2}{T}$ doubling conditions,
we obtain the following bound.

\begin{lemma}
\label{lem:number_of_epochs}
There are $m(T) \leq SAH\log_2(2T)$ global epochs.
\end{lemma}

\textbf{The second concept: optional skipping for estimated transition kernels.}
The trick detailed here is standard in the bandit and reinforcement-learning literature.
The original reference is Theorem~5.2 of \citet[Chapter III, p. 145]{doob1953};
one can also check \citet[Section~5.3]{CT88} for a more recent reference.
A pedagogical exposition of the trick and of its uses in the bandit literature may be found in
\citet[Section~4.1]{KLUCB}, which we adapt now to the setting of reinforcement learning.

For each triplet $(h,s,a) \in [H] \times \cS \times \cA$
and each integer $j \geq 1$, we denote by
\[
N_{h,s,a,j} \eqdef \inf \bigl\{ t \geq 1 : n_{t,h}(s,a) = j \bigr\}
\]
(with the convention that the infimum of an empty set equals $+\infty$)
the predictable stopping time whether and when $(s,a)$ occurs for the $j$--th time.
We are interested in the distribution of the states $s_{t,h+1}$
drawn at rounds $t$ when $(s_{t,h},a_{t,h}) = (s,a)$; these rounds
are given by the stopping times $N_{h,s,a,j}$ introduced above.
It turns out that these states are i.i.d.\ with distribution $P_{h}(\,\cdot \mid s,a)$.
We also have independence across sequences of states.
All these results are formally stated in the following lemma: to do so,
one needs to set the values of the number of times $n_{t,h}(s,a)$
each triplet $(h,s,a)$ was encountered till a given round.

\begin{lemma}[Doob's optional skipping]
\label{lm:Doob}
Fix $t \geq 1$ and consider sequences of integers $J_{h,s,a} \geq 1$ and
the intersection of events
\[
\cC = \bigcap_{h \in [H-1]} \bigcap_{(s,a) \in \cS \times \cA} \bigl\{ n_{t,h}(s,a) = J_{h,s,a} \bigr\}\,.
\]
It holds that on $\cC$, each of the sequences
\[
(\ts_{h,s,a,j})_{j \in [J_{h,s,a}]} \eqdef
\bigl( s_{N_{h,s,a,j},h+1} \bigr)_{j \in [J_{h,s,a}]}
\]
is formed by i.i.d.\ variables, with common distribution $P_{h}(\,\cdot \mid s,a)$. In addition,
these sequences are independent from each other as $(h,s,a)$ vary in $[H-1] \times \cS \times \cA$.
\end{lemma}

One of our applications of Lemma~\ref{lm:Doob} will be the following, to handle term $\termA$.
The proof consists of noting first that on $\bigl\{ \ell_{t-1,h}(s,a) = \ell \bigr\}$,
the distribution $\hatP_{t,h}(\,\cdot \mid s,a)$ corresponds to the empirical measure
of the i.i.d.\ variables $\ts_{h,s,a,j}$ with $1 \leq j \leq 2^{\ell-1}$,
and second, by dropping the indicator function.

Notation-wise, we will be using $\ts_{h,s,a,j}$ (as in Lemma~\ref{lm:Doob}) for random variables
generated by the MDP interactions and $\sigma_{h,s,a,j}$ (as in Corollary~\ref{cor:Doob})
for random variables independent from everything else and that are representations of the former.

\begin{corollary}
\label{cor:Doob}
Fix $h \in [H-1]$ and $(s,a) \in \cS \times \cA$ and
let $(\sigma_{h,s,a,j})_{j \geq 1}$ be a sequence of i.i.d.\ variables with distribution $P_{h}(\,\cdot \mid s,a)$.
For all functions $\psi \colon \R \to [0,+\infty)$, all functions $g \colon \cS \to \R$,
all integers $\ell \geq 1$,
\begin{align*}
\E \biggl[ \psi \Bigl( \hatP_{t,h} \bullet g(s,a) \Bigr) &\,\, \ind\bigl\{ \ell_{t-1,h}(s,a) = \ell \bigr\} \biggr]\\
&\leq \E \! \left[ \psi \biggl( \frac{1}{2^{\ell-1}} \sum_{j=1}^{2^{\ell-1}} g(\sigma_{h,s,a,j}) \biggr) \right].
\end{align*}
\end{corollary}

\subsection{Term $\termA$: Optimism}
\label{subsec:termA}

Term~$\termA$ is handled thanks to a result already present in the analysis of the \texttt{UCBVI} algorithm~\citep[Lemma 18]{AGM17},
relying on an induction, and thanks to applications of Hoeffding's inequalities together with optional skipping.
Appendix~\ref{app:A} provides the (straightforward) details of the proof of the following lemma.

\begin{restatable}{lemma}{lmtermA}
\label{lem:optimism}
With probability at least $1-\delta$, for all $t\in[T]$ and all $(s,a,h) \in \cS \times \cA \times [H]$,
\begin{align*}
    Q^{\bpi^\star, \br_t, \bP}_h(s,a) &\leq Q^{\bpi^\star, \br_t + \bb_t, \hatbP_t}_h(s,a)\,,\\
V^{\bpi^\star, \br_t, \bP}_h(s) &\leq V^{\bpi^\star, \br_t + \bb_t, \hatbP_t}_h(s)\,.
\end{align*}
Specifically, with probability $\geq 1-\delta$, we have $\termA \leq 0$.
\end{restatable}
As in the original proof of~\citet[Lemma 18]{AGM17}, the result relies only on the concentration of the estimated transition kernel in the direction of \( V^{\bpi^\star, \br_t, \bP}_h \) for any \( (t, h) \in [T] \times [H] \). Crucially, it assumes an oblivious adversary, allowing \( V^{\bpi^\star, \br_t, \bP}_h \) to be treated as deterministic. If the adversary depended on past actions and states, this assumption would break. The issue could then be addressed by concentrating \( \hatbP_t \) around \( \bP \) in \( \ell_1 \)--norm, incurring an additional \( \sqrt{S} \) factor in \( \bb_t \) and propagating to the final regret via $\termD$.

\subsection{Term $\termD$: Bonus Summation}
\label{subsec:termD}

Without the doubling trick, the exploration bonuses summed up along the trajectory can be classically bounded by a $O(\sqrt{T})$ term. The doubling trick introduces only minor changes to this classical step.
Appendix~\ref{app:D} provides the (straightforward) details of the proof of the following lemma,
based on the Hoeffding--Azuma inequality together with simple controls of the form,
for all $h \in [H-1]$ and $(s,a) \in \cS \times \cA$,
\begin{align*}
\sum_{t=1}^T \frac{1}{\sqrt{n_{t,h}(s,a)}}
\, \ind&\big\{(s_{t,h,}, a_{t,h}) = (s,a)\big\} \, \ind\big\{n_{t,h}(s,a) \geq 2\big\}\\
&= \sum_{n=2}^{n_{T,h}(s,a)} \frac{1}{\sqrt{n}}
\leq 2 \sqrt{n_{T,h}(s,a)}\,.
\end{align*}

\begin{restatable}{lemma}{lmtermD}
\label{lem:termD_main}
With probability at least $1-\delta$, we have
\begin{align*}
\termD \leq 7\sqrt{H^4 SAT \ln(2SATH\log_2(2T)/\delta)} + H^2SA\, .
\end{align*}
\end{restatable}

\subsection{Term $\termC$: Concentration}
\label{subsec:termC}

Let us start by formally stating the result, whose
detailed proof may be found in Appendix~\ref{app:C};
below, we only sketch that proof.
The analysis is essentially borrowed from~\citet{zhang2023settling} with minor technical modifications
but a much simplified exposition (as we do not target optimized bounds yet).
Also, we explain in Remark~\ref{rk:H5} of Appendix~\ref{app:C}
that the dependency in $H$ of the leading term in the upper bound
of Lemma~\ref{lem:termC_main} could be improved to $\sqrt{H^6}$
with more efforts, but that there is no point in doing so,
given the bound of Lemma~\ref{lem:termB_main}, which
scales with $H$ as $\sqrt{H^7}$.

\begin{restatable}{lemma}{lmtermC}
\label{lem:termC_main}
With probability at least $1-\delta$,
\begin{align*}
\termC \leq
&\ 2\sqrt{H^7 SA \, T \bigl( \log_2(2T) \bigr)^3} \\
&+
2\sqrt{2 H^6 \, T \log_2(2T) \ln(4H/\delta)} + SA H^3\,.
\end{align*}
\end{restatable}

\begin{proof}[Proof sketch]
An application of the performance-difference lemma in case of
different transition kernels (see, e.g., \citealp[Lemma 3]{russo2019worst})
together with the Hoeffding--Azuma inequality first shows that with probability
at least $1-\delta/2$, term $\termC$ equals
\begin{align*}
\sum_{t=1}^T \E\!& \left[ \sum_{h=1}^{H-1} \bigl( \hatP_{t, h} {-} P_h \bigr)
\bullet V_{h+1}^{\bpi_t, \br_t {+} \bb_t, \hatbP_{t}}(s_{t,h}{,} a_{t,h}) \,\bigg|\, \bpi_t{,} \bb_t{,} \hatbP_{t} \right] \\
& \leq \sum_{h=1}^{H-1}
\underbrace{\sum_{t=1}^T \bigl( \hatP_{t, h} - P_h \bigr) \bullet V_{h+1}^{\bpi_t, \br_t + \bb_t, \hatbP_{t}}(s_{t,h}, a_{t,h})}_{= \, \xi_{T,h}}\\
&\phantom{\leq a}+ \sqrt{2 H^5 T \ln(2/\delta)}\,.
\end{align*}
We bound the quantities $\xi_{T,h}$ for each fixed $h \in [H-1]$.
We apply optional skipping in a careful way on the event
$\cC_{\ell,j,h,s,a,t}$ when $(s,a) \in \cS \times \cA$ is played
for the $(2^{\ell-1} + j)$--th time in stage~$h$ at episode $t$: on $\cC_{\ell,j,h,s,a,t}$,
\begin{align*}
& \bigl( \hatP_{t, h} - P_h \bigr) \bullet V_{h+1}^{\bpi_t, \br_t + \bb_t, \hatbP_{t}}(s_{t,h}, a_{t,h}) \quad\mbox{behaves like}\\
& \frac{1}{2^{\ell-1}} \sum_{j \in [2^{\ell-1}]} \bigl( \tV_{s,a,h+1}(\sigma_{h,s,a,j}) -
P_h \tV_{s,a,h+1}(s,a) \bigr)\,,
\end{align*}
for some random variable $\tV_{s,a,h+1}$,
where the $\sigma_{h,s,a,j}$ are i.i.d.\ according to $P_h(\,\cdot \mid s,a)$
and are independent from $\tV_{s,a,h+1}$.
The argument also
extends between pairs $(s,a)$ so that a careful application of the Hoeffding-Azuma
inequality (this is the delicate part of the proof), together with the consideration of all
values for $\ell$ and $j$, then shows that, with probability
at least $1-\delta/2$, we have $\xi_{T,h} \leq$ \\
\ \\[-1.65cm]

\[
\sum_{\ell=1}^{\lceil \log_2 T \rceil}
\sum_{j=1}^{2^{\ell-1}} \sqrt{\frac{2H^4}{2^{\ell-1}} \sum_{(s,a)} \ind\bigl\{ n_{T,h}(s,a) \geq 2^{\ell-1} + j \bigr\}
\ln\frac{1}{\delta'}}\,, \vspace{-.2cm}
\]

where $\delta'$ equals $\delta/2$ divided by the number of times we applied the union bound
over $\ell$, $j$, $H$, and in the course of optional skipping; we bound this number of times
by $4 H (T+1)^{1+SAH \ilog{2}{T}}$.

We conclude the proof by two consecutive applications of Jensen's inequality for the root:
\begin{multline*}
\hspace{-.5cm} \sum_{\ell=1}^{\lceil \log_2 T \rceil}
\sum_{j=1}^{2^{\ell-1}} \sqrt{\frac{2H^4}{2^{\ell-1}} \sum_{(s,a)} \ind\bigl\{ n_{T,h}(s,a) \geq 2^{\ell-1} + j \bigr\}
\ln\frac{1}{\delta'}} \leq \\
\hspace{-.25cm} \sqrt{\vphantom{\sum_{1}^{\lceil \log_2}}2H^4 \lceil \log_2 T \rceil \smash{\underbrace{\sum_{(s,a)} \! \sum_{\ell=1}^{\lceil \log_2 T \rceil}
\sum_{j=1}^{2^{\ell-1}} \ind\bigl\{ n_{T,h}(s,a) {\geq} 2^{\ell-1} {+} j \bigr\} }_{\leq \, \sum_{(s,a)} n_{T,h}(s,a) \,
= \, T} \ln\frac{1}{\delta'}} \vspace{-2.5cm}}\\
\end{multline*}

together with some algebra.
\end{proof}

\section{Conclusion and Limitations}
\label{sec:discussion_and_challenges}

In this work, we proposed an algorithm called \texttt{APO-MVP}
that extends algorithm \texttt{MVP} of \citet{zhang2023settling} and its analysis
to the case of adversarial reward functions, thanks, in particular, to a black-box adversarial aggregation
mechanism due to~\citet{JMS23} that takes care of the adversarial nature of reward functions.
Algorithm \texttt{APO-MVP} is easy to implement in practice as it relies on OLO learning strategies in the policy space combined with dynamic programming;
it does not at all rely on so-called occupancy measures.
Furthermore, it achieves a better regret bound compared to previous approaches based on the occupancy measures, reducing their regret bounds by a
$\sqrt{S}$ multiplicative factor and narrowing the gap between the adversarial and stochastic regret bounds, which are both shown
to be of order $\sqrt{SAT}$ up to logarithmic factors, as far as dependencies on $S$, $A$, and $T$ are concerned.

We believe that this work opens many interesting follow-up questions. The two main open questions are inevitably linked with the main limitations of this paper and are discussed below.

\textbf{Limitations.} The main limitation is rather a high dependency of our regret bound
on the length $H$ of the episodes, of order $\sqrt{H^7}$. Improving this dependency while maintaining a regret of order $\sqrt{SAT}$ up to logarithmic terms
is one of the remaining open questions.
We have only considered full monitoring;
extending our approach to bandit monitoring seems to be non-trivial
and it is still unknown if a $\sqrt{SAT}$
regret
is possible in the adversarial case with bandit feedback.

\section*{Acknowledgments}
The work of Daniil Tiapkin has been supported by the Paris Île-de-France Région in the framework of DIM AI4IDF.

\bibliographystyle{plainnat}
\bibliography{TCS--RL-dependency-S-Bib.bib}

\clearpage
\section*{Checklist}



 \begin{enumerate}

 \item For all models and algorithms presented, check if you include:
 \begin{enumerate}
   \item A clear description of the mathematical setting, assumptions, algorithm, and/or model. [Yes, see Section~\ref{sec:setup} for the mathematical setting, Section~\ref{sec:defalgo} and Appendix~\ref{app:A} for the algorithm description.]
   \item An analysis of the properties and complexity (time, space, sample size) of any algorithm. [Yes, see Section~\ref{sec:defalgo} for the regret analysis and Appendix~\ref{app:A} for detailed analysis of computational and space complexities.]
   \item (Optional) Anonymized source code, with specification of all dependencies, including external libraries. [Not Applicable, no empirical results.]
 \end{enumerate}

 \item For any theoretical claim, check if you include:
 \begin{enumerate}
   \item Statements of the full set of assumptions of all theoretical results. [Yes, see Section~\ref{sec:setup} and Section~\ref{sec:defalgo}.]
   \item Complete proofs of all theoretical results. [Yes, see Section~\ref{sec:decomp} as well Appendices~\ref{app:B}-\ref{app:D}.]
   \item Clear explanations of any assumptions. [Yes, see Section~\ref{sec:setup}.]
 \end{enumerate}

 \item For all figures and tables that present empirical results, check if you include:
 \begin{enumerate}
   \item The code, data, and instructions needed to reproduce the main experimental results (either in the supplemental material or as a URL). [Not Applicable,  no experimental results.]
   \item All the training details (e.g., data splits, hyperparameters, how they were chosen). [Not Applicable]
         \item A clear definition of the specific measure or statistics and error bars (e.g., with respect to the random seed after running experiments multiple times). [Not Applicable]
         \item A description of the computing infrastructure used. (e.g., type of GPUs, internal cluster, or cloud provider). [Not Applicable]
 \end{enumerate}

 \item If you are using existing assets (e.g., code, data, models) or curating/releasing new assets, check if you include:
 \begin{enumerate}
   \item Citations of the creator If your work uses existing assets. [Not Applicable,  no empirical results.]
   \item The license information of the assets, if applicable. [Not Applicable]
   \item New assets either in the supplemental material or as a URL, if applicable. [Not Applicable]
   \item Information about consent from data providers/curators. [Not Applicable]
   \item Discussion of sensible content if applicable, e.g., personally identifiable information or offensive content. [Not Applicable]
 \end{enumerate}

 \item If you used crowdsourcing or conducted research with human subjects, check if you include:
 \begin{enumerate}
   \item The full text of instructions given to participants and screenshots. [Not Applicable]
   \item Descriptions of potential participant risks, with links to Institutional Review Board (IRB) approvals if applicable. [Not Applicable]
   \item The estimated hourly wage paid to participants and the total amount spent on participant compensation. [Not Applicable]
 \end{enumerate}

 \end{enumerate}

\clearpage
\onecolumn
\appendix

\section{Detailed Algorithm Description}
\label{sec:algorithm_app}

In this section, we first provide a fully-detailed algorithmic description of the
strategy introduced in Section~\ref{sec:defalgo} (which had been summarized
in an algorithm sketch titled Algorithm~\ref{alg:main_short}), and
we then write closed-form expressions of the
policy constructions~\ref{eq:policies} based on the strategies of
Examples~\ref{ex:polpot} and~\ref{ex:adahedge}, as well as AdaHedge.

\begin{algorithm}[h]
    \KwData{Number of rounds $T$, number of states, actions and horizon $S,A,H$, confidence level $\delta$,
    online linear optimization strategy $\varphi = (\varphi_t)_{t \geq 1}$ \smallskip}
    \KwResult{Sequence of policies $\bpi^t = (\pi_{t,h})_{h \in [H]}$, for $t \in [T]$ \smallskip}
    Initialize kernels $\hatP_h(s' | s,a) = 1/S$ for all $(s',s,a,h) \in \cS \times \cS \times \cA \times [H-1]$\;
    Initialize counters $n_h(s,a,s') = n_h(s,a) = 0$ for all $(s',s,a,h)\in \cS \times \cS \times \cA \times [H-1]$\;
    Initialize histories $\cH_{h,s} = \emptyset$ for all $h \in [H]$ and $s \in \cS$\;
    Initialize $\pi_{1, h}(\,\cdot \mid s) = \phi_1(\emptyset)$ for all $h \in [H]$ and $s \in \cS$\;
    Select initial state $s_1 \in \cS$\; \smallskip

    \For{rounds $t = 1,\ldots, T$}{
        \tcc{Interaction}
        Set $s_{t,1} = s_1$\;
        \For{$h=1,\ldots,H$}{
            Play action $a_{t,h} \sim \pi_{t,h}(\,\cdot \mid s_{t,h})$\;
            Receive next state $s_{t,h+1} \sim P_h(\,\cdot \mid s_{t,h}, a_{t,h})$\;
            Update counters $n_h(s_{t,h}, a_{t,h}) \pluseq 1 $ and $n_h(s_{t,h}, a_{t,h}, s_{t,h+1}) \pluseq 1$\;
            \tcc{Trigger, update the model}
            \If{$n_h(s_{t,h}, a_{t,h}) = 2^{\ell-1}$ for some $\ell \geq 1$}{
                $\displaystyle{\hatP_h(s' | s_{t,h}, a_{t,h}) = \frac{n_h(s_{t,h}, a_{t,h}, s')}{2^{\ell-1}}}$ for all $s' \in \cS$\;
                $\displaystyle{b_{h}(s_{t,h}, a_{t,h}) = \sqrt{\frac{2H^2\ln(2SATH \log_2(2T) / \delta)}{2^{\ell - 1}}} \wedge H}$\;
                Activate \texttt{trigger}\;
            }
        }
        Receive a reward function $\br_t = (r_{t, h})_{h \in [H]}$\;
        \eIf{\texttt{trigger}}{
            Drop all histories, i.e., set $\cH_{h,s} = \emptyset$ for all $h \in [H]$ and $s \in \cS$\;
            Set $\pi_{t+1, h}(\,\cdot \mid s) = \phi_t(\emptyset)$ for all $h \in [H]$ and $s \in \cS$\;
            Deactivate \texttt{trigger}\;
        }{
        \tcc{Compute first the advantage functions via Bellman's equations}
        Let $\hat{Q}_{H}(s,a) = r_{t,H}(s,a)$ and $\hat{V}_{h}(s) = \pi_{t,H} \bullet \hat{Q}_{H}(s)$ for each $(s,a) \in \cS \times \cA$\;
        \For{$h = H-1, \, H-2, \, \ldots, \, 1$}{
            Let $\hat{Q}_{h}(s,a) = r_{t,h}(s,a) + b_{h}(s,a) + \hatP_h \bullet \hat{V}_{h+1}(s,a)$
            and $\hat{V}_{h}(s) = \pi_{t,h} \bullet \hat{Q}_{h}(s)$ \\ ~~~~for each $(s,a) \in \cS \times \cA$\;
        }
        Let $\hat{A}_{h}(s,a) = \hat{Q}_{h}(s,a) - \hat{V}_{h}(s)$ for all $(s,a,h) \in \cS \times \cA \times [H]$\;
        Add $\bigl( \hat{A}_{h}(s,a) \bigr)_{a \in \cA}$ to the history $\cH_{h,s}$ for each $(h,s) \in [H] \times \cS$\;
        \tcc{Next, obtain $\bpi_{t+1}$ via the learning strategy}
        Let $\pi_{t+1,h}(\,\cdot \mid s) = \phi_t(\cH_{h,s})$ for all $(h,s) \in [H] \times \cS$\;
        }
    }
    \caption{\label{alg:main} Adversarial Policy Optimization based on Monotonic Value Propagation (\texttt{APO-MVP})}
\end{algorithm}

\paragraph{Closed-form expressions of the policy constructions.}
We first recall the statement~\eqref{eq:policies} for the construction of policies:
for all $t \geq 1$, all $h \in [H]$, and $s \in \cS$,
\[
\pi_{t,h}(\,\cdot \mid s) = \varphi_t \biggl( \Bigl( \hatA_{\tau,h}(s,\,\cdot\,)\Bigr)_{\tau \in \cE_{e_t} \cap [t-1]} \biggr)\,.
\]
We now illustrate this definition with the strategies
of Examples~\ref{ex:polpot} and~\ref{ex:adahedge}, as well as with \texttt{AdaHedge}.
A key observation to do so will be that, by definition of advantage functions
and since $\hatA_{\tau,h}(s,\,\cdot\,)$ is based on the policy~$\pi_{\tau,h}$,
\[
\forall \tau \in [T], \ \
\forall s \in \cS, \qquad
\sum_{a \in \cA} \pi_{\tau,h}(a \mid s) \, \hatA_{\tau,h}(s,a) = 0\,.
\]

\paragraph{Polynomial potential (Example~\ref{ex:polpot}).}
We denote by $(x)_+ = \max\{ x, \, 0\}$ the non-negative part of $x \in \R$.
We have $\varphi_1 \equiv (1/A, \, \ldots, 1/A)$ and for $t \geq 2$,
whenever $\cE_{e_t} \cap [t-1]$ contains at least one element,
\begin{align*}
\pi_{t,h}(a \mid s) & =
\frac{\displaystyle{\left( \sum_{\tau \in \cE_{e_t} \cap [t-1]}
\hatA_{\tau,h}(s,a) - \sum_{a'' \in \cA} \pi_{\tau,h}(a'' \mid s) \, \hatA_{\tau,h}(s,a'') \right)_{\!\!+}^{\! 2 \ln(A)}}}{
\displaystyle{\sum_{a' \in \cA} \left( \sum_{\tau \in \cE_{e_t} \cap [t-1]}
\hatA_{\tau,h}(s,a') - \sum_{a'' \in \cA} \pi_{\tau,h}(a'' \mid s) \, \hatA_{\tau,h}(s,a'') \right)_{\!\!+}^{\! 2 \ln(A)}}
} \\
& =
\frac{\displaystyle{\left( \sum_{\tau \in \cE_{e_t} \cap [t-1]}
\hatA_{\tau,h}(s,a) \right)_{\!\!+}^{\! 2 \ln(A)}}}{\displaystyle{\sum_{a' \in \cA} \left( \sum_{\tau \in \cE_{e_t} \cap [t-1]}
\hatA_{\tau,h}(s,a') \right)_{\!\!+}^{\! 2 \ln(A)}}}\,.
\end{align*}

\paragraph{Exponential potential (Example~\ref{ex:adahedge}).}
Similarly to above,
we have $\varphi_1 \equiv (1/A, \, \ldots, 1/A)$ and for $t \geq 2$,
whenever $\cE_{e_t} \cap [t-1]$ contains at least one element,
\[
\pi_{t,h}(a \mid s) =
\frac{\displaystyle{\exp \! \left( \eta_t \sum_{\tau \in \cE_{e_t} \cap [t-1]}
\hatA_{\tau,h}(s,a) \right)}}{
\displaystyle{\sum_{a' \in \cA} \exp \! \left( \eta_t \sum_{\tau \in \cE_{e_t} \cap [t-1]}
\hatA_{\tau,h}(s,a') \right)}
} \qquad \mbox{where} \qquad
\eta_t = \frac{1}{H+1} \sqrt{\frac{\ln(A)}{\bigl| \cE_{e_t} \cap [t-1] \bigr| }}
\]
are time-varying learning rates, based on the cardinality $\bigl| \cE_{e_t} \cap [t-1] \bigr|$
of $\cE_{e_t} \cap [t-1]$.

\paragraph{Adaptive versions of exponential-potential-based strategies.}
The literature proposed many ways of setting the learning rates for exponential
potentials based on past information---a series of work initiated by~\citet{ACBG02},
whose learning rates were used in the paragraph above.
One may cite, among (many) others, \citet{SOE},
\citet{EKRG11}, \citet{JMLR:v15:rooij14a}, \citet{OP15}; sometimes,
the resulting strategy is called \texttt{AdaHedge}.
For instance, \citet[Section~7.6]{O19} summarizes this literature by
the following learning rates:
\[
\eta_t = \frac{\max\{4, 2^{-1/4}\sqrt{\ln(A)}\}}{\sqrt{\displaystyle{\sum_{\tau \in \cE_{e_t} \cap [t-1]} \,\,
\max_{a' \in \cA} \bigl( \hatA_{\tau, h}(s, a') \bigr)^2}}}\,.
\]
These updates correspond to an OLO strategy satisfying the bound of Definition~\ref{def:adv_main}
with a performance bound $B_{T, K} = 4\sqrt{T \ln(K)}$.

\paragraph{Space and computational complexity.}
We begin by outlining the space and computational complexity of the policy optimization strategy discussed above. In both examples (Example~\ref{ex:polpot}-\ref{ex:adahedge}), it is clear that the computation of the policy relies solely on the sum of advantage estimates over the current epoch. These can be stored with a space complexity of $\mathcal{O}(HSA)$, and the computational complexity of updating the policy is $\mathcal{O}(HSA)$ per step, as it involves two key operations: (1) updating the sum of advantage estimates, and (2) updating the policy via a closed-form formula.

Now, the overall computations for one iteration of the final algorithm, outlined in Algorithm~\ref{alg:main}, consists of two main phases:
(1) the dynamic programming update, which requires $\mathcal{O}(S^2AH)$ calculations in the worst case, and
(2) the policy optimization update, which requires $\mathcal{O}(SAH)$ calculations, as previously described.
The space complexity of the algorithm is standard for model-based reinforcement learning algorithms and is equal to $\mathcal{O}(S^2AH)$.

In summary, the algorithm introduces no additional computational and space overhead beyond that of standard dynamic programming methods.

\section{Term~$\termB$}
\label{app:B}

It only remains to prove Lemma~\ref{lem:OLO_main}, which we restate below.
For the sake of self-completeness, we copy the proofs by \citet[in the $\gamma$--discounted setting]{JMS23}
and \citet[in the same $H$--episodic setting as the present article]{JMS25}.
As noted by \citet[Remark~1]{JMS25}, the proof below
may be adapted to strategies run with $\Q$--functions rather than with advantage functions:
\eqref{eq:csq-agreg-A} is then replaced by
\[
2M(H-h+1)\,B_{T,A} \geq
\max_{a \in \cA} \sum_{t=1}^T Q^{\bpi'_t,\br'_t,\bP'}_{h}(s,a)
- \sum_{t=1}^T \underbrace{\sum_{a \in \cA} \pi'_{t,h}(a|s) \, Q^{\bpi'_t,\br'_t,\bP'}_{h}(s,a)}_{= V^{\bpi'_t,\br'_t,\bP'}_{h}(s)}
=  \max_{a \in \cA} \sum_{t=1}^T A^{\bpi'_t,\br'_t,\bP'}_{h}(s,a) \,,
\]
and the rest of the proof is unchanged. However, advantage functions are preferred in practice.

\lmaggreg*

\begin{proof}
As the reward function takes values in~$[0,M]$, we have that $\bigl| A^{\bpi'_\tau}_{\tau,h}(s,a) \bigr| \leq M(H-h+1)$.
By the definition of advantage functions (for the equality to~$0$) and by
Definition~\ref{def:adv_main} (for the upper bound), we have, for all $s \in \cS$,
\begin{equation}
\label{eq:csq-agreg-A}
\max_{a \in \cA} \sum_{t=1}^T A^{\bpi'_t,\br'_t,\bP'}_{h}(s,a)
- \sum_{t=1}^T \underbrace{\sum_{a \in \cA} \pi'_{t,h}(a|s) \, A^{\bpi'_t,\br'_t,\bP'}_{h}(s,a)}_{=0}
\leq 2M(H-h+1)\,B_{T,A}\,.
\end{equation}
Now, the so-called performance difference lemma (see, e.g., \citealp{kakade2002approximately} for the result in the discounted setting) shows that
\[
V_1^{\bpi,\br'_t,\bP'}(s_1) - V_1^{\bpi'_t,\br'_t,\bP'}(s_1)
= \sum_{h=1}^H \sum_{s \in \cS} \mu_h^{\bpi,\bP',s_1}(s) \sum_{a \in \cA} \pi_h(a|s) \, A^{\bpi'_t,\br'_t,\bP'}_{h}(s,a)
\]
where $\mu_h^{\bpi,\bP',s_1}$ is the distribution of $s_{t,h}$
induced in the $h$--th episode by $\bpi$ given the state transitions $\bP'$ and the initial state $s_1$.
Summing this equality over $t$ and rearranging, we get
\begin{align*}
\sum_{t=1}^T \Bigl( V_1^{\bpi,\br'_t,\bP'}(s_1) - V_1^{\bpi'_t,\br'_t,\bP'}(s_1) \Bigr)
& = \sum_{h=1}^H \sum_{s \in \cS} \mu_h^{\bpi,\bP',s_1}(s) \sum_{a \in \cA} \pi_h(a|s) \sum_{t=1}^T A^{\bpi'_t,\br'_t,\bP'}_{h}(s,a) \\
& \leq \sum_{h=1}^H \sum_{s \in \cS} \mu_h^{\bpi,\bP',s_1}(s) \,\, \underbrace{\max_{a \in \cA} \sum_{t=1}^T A^{\bpi'_t,\br'_t,\bP'}_{h}(s,a)}_{\leq
2M(H-h+1)\,B_{T,A}} \\
&\leq 2M \sum_{h=1}^H (H-h+1) \, B_{T,A} = MH(H+1) \, B_{T,A} \leq 2MH^2 \, B_{T,A}\,,
\end{align*}
where we substituted~\eqref{eq:csq-agreg-A}. Here, we crucially used that the convex combination with weights
$\mu_h^{\bpi,\bP',s_1}(s)$ is independent of $t$ and only depends on the fixed benchmark policy~$\bpi$,
on the state transitions $\bP'$, and on the initial states $s_1$ (identical for all $t$).
\end{proof}

\section{Term~$\termA$}
\label{app:A}

We start with the following consequence of Hoeffding's inequality.

\begin{lemma}
\label{lem:concentration_event_optimism}
For each $t \in [T]$, for each $h \in [H-1]$,
for each $(s,a) \in \cS \times \cA$, for each $\ell \in \bigl[ \lceil \log_2(T) \rceil \bigr]$,
\begin{multline*}
\hspace{-.2cm}
\P \left\{ \Biggl| \frac{1}{2^{\ell-1}} \sum_{j=1}^{2^{\ell-1}} V^{\bpi^\star, \br_t, \bP}_{h+1}(\sigma_{h,s,a,j}) -
P_h \bullet V^{\bpi^\star, \br_t, \bP}_{h+1}(s,a) \Biggr| > \sqrt{\frac{2H^2 \ln(2SATH\log_2(2T)/\delta)}{2^{\ell-1}}} \wedge H \right\} \\
 \leq \frac{\delta}{SATH\log_2(2T)}\,,
\end{multline*}
where $(\sigma_{h,s,a,j})_{1 \leq j \leq 2^{\ell-1}}$ is a sequence of i.i.d.\ variables with distribution $P_{h}(\,\cdot \mid s,a)$.
\end{lemma}

\begin{proof}
The policy $\bpi^\star$ is fixed, as it only depends on the $\br_t$ and $\bP$, which are all fixed beforehand.
The function $g = V^{\bpi^\star, \br_t, \bP}_{h+1}$ is therefore a fixed deterministic function.
The expectation of $g(\sigma_{h,s,a,j})$ is indeed, given our notation, $P_h \bullet g$.
By the boundedness of rewards in $[0,1]$, and thus the boundedness of values in the range $[0,H]$, we
may therefore apply Hoeffding's inequality: we do so for each $t \in [T]$, each $(s,a,h) \in \cS \times \cA \times [H-1]$, and
each $\ell \in \bigl[ \lceil \log_2(T) \rceil \bigr]$, and get that for all $\delta' \in (0,1)$, with probability at least $1-\delta'$,
\[
\Biggl| \frac{1}{2^{\ell-1}} \sum_{j=1}^{2^{\ell-1}} V^{\bpi^\star, \br_t, \bP}_{h+1}(\sigma_{h,s,a,j}) -
P_h \bullet V^{\bpi^\star, \br_t, \bP}_{h+1}(s,a) \Biggr|
\leq \sqrt{\frac{2H^2 \ln(2/\delta')}{2^{\ell-1}}}\,.
\]
The proof is concluded by keeping in mind that the left-hand side necessarily belongs to $[0,H]$
by boundedness of values in $[0,H]$.
\end{proof}

We are now ready to prove Lemma~\ref{lem:optimism}, which we restate below; we
do so by mimicking the proof of \citet[Lemma 18]{AGM17}.

\lmtermA*

\begin{proof}
We proceed by backward induction, for each given $t \in [T]$;
more precisely, we consider, for $h \in [H]$, the induction hypothesis
\begin{align}
\label{eq:cHh}
\tag{$\cH_h$}
\forall (s,a) \in \cS \times \cA,
\qquad \qquad \qquad & Q^{\bpi^\star, \br_t, \bP}_h(s,a) \leq Q^{\bpi^\star, \br_t + \bb_t, \hatbP_t}_h(s,a) \\
\nonumber
\text{and}\qquad
& V^{\bpi^\star, \br_t, \bP}_h(s) \leq V^{\bpi^\star, \br_t + \bb_t, \hatbP_t}_h(s)\,.
\end{align}
For $h = H$, we note that for all $(s,a)$,
\[
Q^{\bpi^\star, \br_t, \bP}_H(s,a) = r_{t, H}(s,a) = r_{t,H}(s,a) + b_{t,H}(s,a) = Q^{\bpi^\star, \br_t + \bb_t, \hatbP_t}_H(s,a)\,,
\]
so that ($\cH_H$) is trivially satisfied.
For $h \in [H-1]$, by Bellman equations,
\[
Q^{\bpi^\star, \br_t + \bb_t, \hatbP_t}_h(s,a) - Q^{\bpi^\star, \br_t, \bP}_h(s,a) =
b_{t,h}(s,a) + \hatP_{t,h} \bullet V^{\bpi^\star, \br_t + \bb_t, \hatbP_t}_{h+1} (s,a) - P_h \bullet V^{\bpi^\star, \br_t, \bP}_{h+1}(s,a) \,,
\]
where by the induction hypothesis ($\cH_{h+1}$), we have $V^{\bpi^\star, \br_t + \bb_t, \hatbP^t}_{h+1}(s') \geq V^{\bpi^\star, \br_t, \bP}_{h+1}(s')$ for any $s'$. Thus,
\begin{equation}\label{eq:optimism_final_term}
Q^{\bpi^\star, \br_t + \bb_t, \hatbP_t}_h(s,a) - Q^{\bpi^\star, \br_t, \bP}_h(s,a) \geq b_{t,h}(s,a) + \bigl( \hatP_{t,h} - P_h \bigr) \bullet
V^{\bpi^\star, \br_t, \bP}_{h+1}(s,a)\,,
\end{equation}
and a similar inequality for values, as the latter are obtained as convex combinations of
$\Q$--values, where the convex weights are determined solely by the common policy $\bpi^\star$ used.
Therefore, ($\cH_{h}$) holds at least on the event
\[
\cG_{t,h} \eqdef \biggl\{ \forall (s,a) \in \cS \times \cA, \quad
\Bigl| \bigl( \hatP_{t,h} - P_h \bigr) \bullet V^{\bpi^\star, \br_t, \bP}_{h+1}(s,a) \Bigr| \leq b_{t,h}(s,a) \biggr\}\,.
\]
All in all, the inequalities required in the statement of the lemma thus hold on the intersection of the events
$\cG_{t,h}$ over $t \in [T]$ and $h \in [H-1]$.

To conclude the proof, it suffices to show that this intersection is of probability at least $1-\delta$.
By considering the complements and by a union bound, it suffices to show that any event
\[
\overline{\cG}_{t,h,s,a} \eqdef \biggl\{
\Bigl| \bigl( \hatP_{t,h} - P_h \bigr) \bullet V^{\bpi^\star, \br_t, \bP}_{h+1}(s,a) \Bigr| > b_{t,h}(s,a) \biggr\}
\]
is of probability at most $\delta / (SATH)$.
We partition the probability space based on the value of $\ell_{t-1,h}(s,a)$,
resort to optional skipping and Corollary~\ref{cor:Doob},
with the deterministic function $g = V^{\bpi^\star, \br_t, \bP}_{h+1}$
(see the proof of Lemma~\ref{lem:concentration_event_optimism}), to get the first inequality below,
and to Lemma~\ref{lem:concentration_event_optimism} for the second inequality below.
We also use the
definition~\eqref{eq:defbe} of $b_{t,h}$:
\begin{align*}
\P & \bigl( \,\overline{\cG}_{t,h,s,a} \bigr)
= \sum_{\ell = 0}^{\lceil \log_2(T) \rceil} \P \Bigl( \overline{\cG}_{t,h,s,a} \cap \bigl\{ \ell_{t-1,h}(s,a) = \ell \bigr\} \Bigr) \\
& = \sum_{\ell = 1}^{\lceil \log_2(T) \rceil} \P \biggl\{
\Bigl| \bigl( \hatP_{t,h} - P_h \bigr) \bullet V^{\bpi^\star, \br_t, \bP}_{h+1}(s,a) \Bigr|
> \sqrt{\frac{2H^2 \ln(2SATH\log_2(2T)/\delta)}{2^{\ell-1}}} \wedge H \quad \mbox{and} \quad \ell_{t-1,h}(s,a) = \ell \, \biggr\} \\
& \leq
\sum_{\ell = 1}^{\lceil \log_2(T) \rceil} \P \left\{ \Biggl| \frac{1}{2^{\ell-1}} \sum_{j=1}^{2^{\ell-1}} V^{\bpi^\star, \br_t, \bP}_{h+1}(\sigma_{h,s,a,j}) -
P_h \bullet V^{\bpi^\star, \br_t, \bP}_{h+1}(s,a) \Biggr| > \sqrt{\frac{2H^2 \ln(2SATH\log_2(2T)/\delta)}{2^{\ell-1}}} \wedge H \right\} \\
& \leq \lceil \log_2(T) \rceil \, \frac{\delta}{SATH\log_2(2T)} \leq \frac{\delta}{SATH}\,,
\end{align*}
where we used the fact that $b_{t,h}(s,a) = H$ is a trivial upper bound on the
difference of values at hand in the case $\ell = 0$, which is why the element $\ell = 0$
gets dropped in the summation in the second equality.
\end{proof}

\section{Term~$\termD$}
\label{app:D}

We first restate and then prove Lemma~\ref{lem:termD_main}.

\lmtermD*

\begin{proof}
Since $b_{t, h}(s, a) \in [0,H]$, the Hoeffding--Azuma inequality implies that with probability at least $1-\delta$,
\begin{equation}
\label{eq:D-init}
\sum_{t=1}^T V^{\bpi_t, \bb_t, \bP}_1(s_1)
- \sum_{t=1}^T \sum_{h \in [H]} b_{t, h}(s_{t,h}, a_{t,h}) \leq \sqrt{\frac{(H^2)^2 \, T \ln(1/\delta)}{2}}\enspace;
\end{equation}
we crucially use here that the policies $\bpi_t$ only depend on information gathered during previous episodes $\tau \leq t-1$
and that the stochastic environment $\bP$ considered in the definition of $\termD$ is the true underlying environment.

We fix $h \in [H-1]$ (recall that $b_{t,H} \equiv 0$) and a pair $(s',a') \in \cS \times \cA$,
and show that
\begin{equation}
\label{eq:D-interm}
\sum_{t=1}^T  b_{t, h}(s_{t,h}, a_{t,h})
\, \ind\big\{(s_{t,h,}, a_{t,h}) = (s',a')\big\} \leq
H + \sqrt{2H^2\ln(2SATH\log_2(2T)/\delta) \, n_{T,h}(s',a')}\,.
\end{equation}
Indeed, there can only be at most one $t$ such $(s_{t,h,}, a_{t,h}) = (s',a')$
and $n_{t,h}(s',a') = 1$; for this $t$, we use the upper bound $b_{t, h}(s_{t,h}, a_{t,h}) \leq H$.
For $t$ such that $n_{t,h}(s',a') \geq 2$, we have,
by the definitions in Section~\ref{sec:defalgo},
that $2^{\ell_{t,h}(s',a')} \geq n_{t,h}(s',a')$ and $\ell_{t-1,h}(s',a')+1 \geq \ell_{t,h}(s',a') \geq \ell_{t-1,h}(s',a')$.
Therefore,
$2^{\ell_{t-1,h}(s',a')-1} \geq 2^{\ell_{t,h}(s',a') -2} \geq n_{t,h}(s',a')/4$.
Substituting this inequality in the definition~\eqref{eq:defbe} of $b_{t,h}$,
we obtain
\begin{multline*}
\sum_{t=1}^T b_{t, h}(s_{t,h}, a_{t,h})
\, \ind\big\{(s_{t,h,}, a_{t,h}) = (s',a')\big\} \\
\leq H + \sum_{t=1}^T \sqrt{\frac{8H^2\ln(2SATH\log_2(2T)/\delta)}{n_{t,h}(s',a')}}
\, \ind\big\{(s_{t,h,}, a_{t,h}) = (s',a')\big\} \, \ind\big\{n_{t,h}(s',a') \geq 2\big\}\,,
\end{multline*}
where, using that the counters $n_{t,h}(s',a')$ vary (by $+1$)
if and only if $(s_{t,h,}, a_{t,h}) = (s',a')$, we also have
\[
\sum_{t=1}^T \frac{1}{\sqrt{n_{t,h}(s',a')}}
\, \ind\big\{(s_{t,h,}, a_{t,h}) = (s',a')\big\} \, \ind\big\{n_{t,h}(s',a') \geq 2\big\}
= \!\!\!\! \sum_{n=2}^{n_{T,h}(s',a')} \frac{1}{\sqrt{n}}
\leq 2\sqrt{n_{T,h}(s',a')}\,.
\]

We conclude the proof by noting first that for each $(s',a') \in \cS \times \cA$,
by concavity of the root,
\[
\sum_{(s',a')} \sqrt{n_{T,h}(s',a')} \leq \sqrt{SAT}\,,
\]
so that summing~\eqref{eq:D-interm} over $h \in [H-1]$ and $(s',a') \in \cS \times \cA$
yields
\[
\sum_{t=1}^T \sum_{h \in [H]} b_{t, h}(s_{t,h}, a_{t,h})
\leq H^2 SA + 4H \sqrt{2H^2 SAT \ln(2SATH\log_2(2T)/\delta)}\,.
\]
We combine this inequality with~\eqref{eq:D-init}
and note that
\[
\sqrt{\frac{H^4 \, T \ln(1/\delta)}{2}} \leq
\sqrt{H^4 SAT \ln(2SATH\log_2(2T)/\delta)}
\]
to get the claimed bound.
\end{proof}

\section{Term~$\termC$}
\label{app:C}

This section is devoted to the analysis of the term $\termC$, which is the most involved part of the proof.
We leverage the recently developed techniques of \citet{zhang2023settling}.
We start by restating the claimed bound.

\lmtermC*

The proof starts with an application of the performance-difference lemma in case of
different transition kernels (see, e.g., \citealp[Lemma 3]{russo2019worst}):
\[
V_1^{\bpi_t, \br_t + \bb_t, \hatbP_{t}}(s_1) - V_1^{\bpi_t, \br_t + \bb_t, \bP}(s_1)
= \E_{\bpi_t, \bP}\!\left[ \sum_{h=1}^{H-1} \bigl( \hatP_{t, h} - P_h \bigr)
\bullet V_{h+1}^{\bpi_t, \br_t + \bb_t, \hatbP_{t}}(s'_{h}, a'_{h}) \right],
\]
where the piece of notation $\E_{\bpi_t, \bP}$ indicates (as in Section~\ref{sec:setup})
that the expectation is taken over trajectories
$(s'_1,a'_1,\,\ldots,\,s'_H,a'_H)$ started at $s'_1 = s_1$ and induced by the policies
$\bpi_t$ and the transition kernels $\bP$. Actually, one such trajectory is exactly
$(s_{t,1},a_{t,1},\,\ldots,\,s_{t,H},a_{t,H})$ and we could rewrite the considered
expectation as a conditional expectation:
\[
\E_{\bpi_t, \bP}\!\left[ \sum_{h=1}^{H-1} \bigl( \hatP_{t, h} - P_h \bigr)
\bullet V_{h+1}^{\bpi_t, \br_t + \bb_t, \hatbP_{t}}(s'_{h}, a'_{h}) \right] = \E\!\left[ \sum_{h=1}^{H-1} \bigl( \hatP_{t, h} - P_h \bigr)
\bullet V_{h+1}^{\bpi_t, \br_t + \bb_t, \hatbP_{t}}(s_{t,h}, a_{t,h}) \,\bigg|\, \bpi_t, \, \bb_t, \, \hatbP_{t} \right].
\]
Next, we apply the Hoeffding--Azuma inequality, by resorting to a lexicographic order on pairs $(t,h)$
and by noting that the random variables at hand satisfy
\[
\bigl( \hatP_{t, h} - P_h \bigr) \bullet V_{h+1}^{\bpi_t, \br_t + \bb_t, \hatbP_{t}}(s_{t,h}, a_{t,h}) \in [-H^2,H^2]\,;
\]
indeed, the sums $r_{t,h}+b_{t,h}$ lies in $[0,H+1]$ and the value functions
are weighted sums of at most $H-1$ such terms. We obtain that with probability at least $1-\delta/2$,
\begin{align}
\nonumber
\termC & = \sum_{t=1}^T V_1^{\bpi_t, \br_t + \bb_t, \hatbP_{t}}(s_1) - V_1^{\bpi_t, \br_t + \bb_t, \bP}(s_1) \\
\label{eq:termC-HAz1}
& \leq \sum_{t=1}^T \sum_{h=1}^{H-1} \bigl( \hatP_{t, h} - P_h \bigr) \bullet V_{h+1}^{\bpi_t, \br_t + \bb_t, \hatbP_{t}}(s_{t,h}, a_{t,h})
+ \sqrt{2 H^5 T \ln(2/\delta)}\,.
\end{align}
We fix $h \in [H-1]$ and use the decomposition
\begin{align}
\nonumber
 \sum_{t=1}^T \bigl( \hatP_{t, h} - P_h \bigr) &\bullet V_{h+1}^{\bpi_t, \br_t + \bb_t, \hatbP_{t}}(s_{t,h}, a_{t,h}) \\
\label{eq:sum-h-fixed}
&\leq  SAH^2 +
\sum_{t=1}^T \sum_{(s,a)}
\sum_{\ell=1}^{\lceil \log_2 T \rceil} \sum_{j=1}^{2^{\ell-1}}
\bigl( \hatP_{t, h} - P_h \bigr) \bullet V_{h+1}^{\bpi_t, \br_t + \bb_t, \hatbP_{t}}(s,a)
\, \ind\bigl\{ (s_{t,h}, a_{t,h}) = (s,a) \bigr\} \\[-.25cm]
\nonumber
& \hspace{7cm} \times \ind\bigl\{ n_{t,h}(s,a) = 2^{\ell-1} + j \bigr\}\,;
\end{align}
the term $SA H^2$ comes from the fact that for each pair $(s,a)$, there is
at most once round $t$ when $(s_{t,h}, a_{t,h}) = (s,a)$ and $ n_{t,h}(s,a) = 1$.
We prove below the following lemma.
\begin{lemma}
\label{lm:main-termC}
For each pair $(\ell,j)$, where $\ell \in \bigl[ \lceil \log_2 T \rceil \bigr]$ and $j \in [2^{\ell-1}]$,
with probability at least $1-\delta/(4TH)$,
\begin{multline*}
\sum_{t=1}^T \sum_{(s,a)} \bigl( \hatP_{t, h} - P_h \bigr) \bullet V_{h+1}^{\bpi_t, \br_t + \bb_t, \hatbP_{t}}(s,a)
\, \ind\bigl\{ (s_{t,h}, a_{t,h}) = (s,a) \bigr\} \, \ind\bigl\{ n_{t,h}(s,a) = 2^{\ell-1} + j \bigr\} \\
\leq
\sqrt{2H^4 \, \frac{1}{2^{\ell-1}} \sum_{(s,a)} \ind\bigl\{ n_{T,h}(s,a) \geq 2^{\ell-1} + j \bigr\}
\ln \biggl( \frac{4 H (T+1)^{1+SAH \ilog{2}{T}}}{\delta} \biggr)}\,.
\end{multline*}
\end{lemma}

\begin{remark}
\label{rk:H5}
Lemma~\ref{lm:main-termC} is established for a fixed $h$, with the derived bound scaling as $\sqrt{H^5}$ in terms of the episode length. Consequently, the final bound on term-$\termC$ scales as $H \sqrt{H^5} = \sqrt{H^7}$ by summing the bound of Lemma~\ref{lm:main-termC} over $h \in [H]$.

A potential refinement of the bound on term $\termC$ involves incorporating the summation over $h$ appearing in~\eqref{eq:termC-HAz1} and proving a strengthened version of Lemma~\ref{lm:main-termC}. In particular, taking the sum over $h$ into account (instead of a fixed value $h$) would replace $\sqrt{H^5}$ in Lemma~\ref{lm:main-termC} by $\sqrt{H^6}$. Since no more summation over $h$ is needed, the final bound would also scale as $\sqrt{H^6}$.

For the sake of simplicity and because the control of term~$\termB$ leads anyway to a factor
$\sqrt{H^7}$ in the final regret bound, we do not pursue this direction.
\end{remark}

We conclude the proof of Lemma~\ref{lem:termC_main} based on Lemma~\ref{lm:main-termC}.
There are at most $2T$ different pairs $(\ell,j)$ considered, so that all events considered
in Lemma~\ref{lm:main-termC} hold simultaneously with probability at least $1-\delta/(2H)$. In addition,
a first application of Jensen's inequality guarantees that for each $1 \leq \ell \leq \lceil \log_2 T \rceil$,
\[
\sum_{j=1}^{2^{\ell-1}}
\sqrt{\frac{1}{2^{\ell-1}} \sum_{(s,a)} \ind\bigl\{ n_{T,h}(s,a) \geq 2^{\ell-1} + j \bigr\}}
\leq \sqrt{\sum_{(s,a)} \sum_{j=1}^{2^{\ell-1}} \ind\bigl\{ n_{T,h}(s,a) \geq 2^{\ell-1} + j \bigr\}}\,,
\]
and a second application yields
\begin{multline*}
\sum_{\ell=1}^{\lceil \log_2 T \rceil}
\sum_{j=1}^{2^{\ell-1}}
\sqrt{\frac{1}{2^{\ell-1}} \sum_{(s,a)} \ind\bigl\{ n_{T,h}(s,a) \geq 2^{\ell-1} + j \bigr\}} \\
\leq
\sqrt{\lceil \log_2 T \rceil
\sum_{(s,a)} \underbrace{\sum_{\ell=1}^{\lceil \log_2 T \rceil} \sum_{j=1}^{2^{\ell-1}}
\ind\bigl\{ n_{T,h}(s,a) \geq 2^{\ell-1} + j \bigr\}}_{= \, n_{T,h}(s,a) - 1}} \leq
\sqrt{T \lceil \log_2 T \rceil}
\,.
\end{multline*}
Therefore,
substituting the bound above (together with Lemma~\ref{lm:main-termC}) into~\eqref{eq:sum-h-fixed},
we proved so far that with probability at least $1-\delta/2$,
\begin{align*}
\sum_{t=1}^T \bigl( \hatP_{t, h} - P_h \bigr) &\bullet V_{h+1}^{\bpi_t, \br_t + \bb_t, \hatbP_{t}} (s_{t,h}, a_{t,h}) \\
& \leq SA H^2 +
\sqrt{2H^4 \, T \lceil \log_2 T \rceil \ln\biggl(\frac{4 H (T+1)^{1+SAH \ilog{2}{T}}}{\delta}\biggr)} \\
& \leq SA H^2 + 2\sqrt{H^5 SA \, T \bigl( \log_2(2T) \bigr)^3} + \sqrt{2H^4 \, T \lceil \log_2 T \rceil \ln(4H/\delta)}\,.
\end{align*}
Summing this bound over $h \in [H-1]$ and combining the outcome with~\eqref{eq:termC-HAz1} leads to
\[
\termC \leq
SA H^3 +
2\sqrt{H^7 SA \, T \bigl( \log_2(2T) \bigr)^3} + \sqrt{2H^6 \, T \lceil \log_2 T \rceil \ln(4H/\delta)}
+ \sqrt{2 H^5 T \ln(2/\delta)}\,,
\]
and thus to the upper bound claimed in Lemma~\ref{lem:termC_main}.

It therefore only remains to prove Lemma~\ref{lm:main-termC}.

\begin{proof}
We denote by
\[
\tau_{\ell,j,h}(s,a) \eqdef
\begin{cases}
t & \mbox{if} \ (s_{t,h},a_{t,h}) = (s,a) \ \mbox{and} \ n_{t,h}(s,a) = 2^{\ell-1} + j\,, \\
+\infty & \mbox{if} \ n_{T,h}(s,a) \leq 2^{\ell-1} + j-1\,,
\end{cases}
\]
the stopping time whether and when $(s,a)$ was reached in stage~$h$ for the $(2^{\ell-1} + j)$--th time,
with the convention $\tau_{\ell,j,h}(s,a) = +\infty$ if $(s,a)$ was reached fewer times than that.
To apply optional skipping, we
will partition the underlying probability space according
to the values of all the $\ell_{t',h'}(s',a')$ as $t',h',s',a'$ vary
and of the $\tau_{\ell,j,h}(s',a')$ as $s',a'$ only vary.

\emph{Part 1: Hoeffding--Azuma inequality.}
We fix consistent sequences $k_{t',h'}(s',a') \in [T]$ and $\kappa_{\ell,j,h}(s',a')$ of values
for the $\ell_{t',h'}(s',a')$ and the $\tau_{\ell,j,h}(s',a')$;
in particular, $k_{\kappa_{\ell,j,h}(s,a)-1,h}(s,a) = \ell$.
The notation in the display below is heavy but the high-level idea is simple to grasp: only
rounds $t = \kappa_{\ell,j,h}(s,a)$ matter, and we know to which global epoch each of these rounds
belongs and, in particular, we know which averages are in the components
$\hatP_{t, h}(\,\cdot \mid s',a')$ of $\hatP_{t, h}$.

We rewrite the quantity at hand on the event associated with the sequences fixed:
\begin{align}
\nonumber
& \sum_{t=1}^T \sum_{(s,a)} \bigl( \hatP_{t, h} - P_h \bigr) \bullet V_{h+1}^{\bpi_t, \br_t + \bb_t, \hatbP_{t}}(s,a)
\, \ind\bigl\{ (s_{t,h}, a_{t,h}) = (s,a) \bigr\} \, \ind\bigl\{ n_{t,h}(s,a) = 2^{\ell-1} + j \bigr\} \\[-.3cm]
\nonumber
& \times \prod_{(s',a') \in \cS \times \cA} \left(
\ind\bigl\{ \tau_{\ell,j,h}(s',a') = \kappa_{\ell,j,h}(s',a') \bigr\} \prod_{t'=1}^T \prod_{h' \in [H-1]}
\ind\bigl\{ \ell_{t',h'}(s',a') = k_{t',h'}(s',a') \bigr\} \right) \\
\nonumber
\leq \ & \sum_{(s,a) : \kappa_{\ell,j,h}(s,a) \leq T} \bigl( \hatP_{\kappa_{\ell,j,h}(s,a), h} - P_h \bigr) \bullet \hatV_{\kappa_{\ell,j,h}(s,a),h+1}(s,a) \\[-.3cm]
\label{eq:C-1}
& \hspace{2.5cm} \times \prod_{(s',a') \in \cS \times \cA} \prod_{h' \in [H-1]}
\ind\bigl\{ \ell_{\kappa_{\ell,j,h}(s,a)-1,h'}(s',a') = k_{\kappa_{\ell,j,h}(s,a)-1,h'}(s',a') \bigr\}\,,
\end{align}
where we used the short-hand notation $\hatV_{t,h+1} \eqdef V_{h+1}^{\bpi_t, \br_t + \bb_t, \hatbP_{t}}$.

We are now ready to apply optional skipping---a concept recalled in Section~\ref{sec:add-concepts},
whose notation we use again here.
On the events
\[
\cC' \eqdef \bigcap_{(s',a') \in \cS \times \cA} \bigcap_{h' \in [H-1]}
\bigl\{ \ell_{\kappa_{\ell,j,h}(s,a)-1,h'}(s',a') = k_{\kappa_{\ell,j,h}(s,a)-1,h'}(s',a') \bigr\}
\]
considered, the empirical averages $\hatP_{\kappa_{\ell,j,h}(s,a), h'}(\,\cdot \mid s',a')$
have the same distributions as the empirical frequency vectors associated with the i.i.d.\ random variables
\[
\sigma_{h',s',a',j}\,, \qquad j \in \bigl[ 2^{k_{\kappa_{\ell,j,h}(s,a)-1,h'}-1} \bigr]\,,
\qquad \mbox{with distribution} \ \ \ P_{h'}(\,\cdot \mid s',a'),
\]
and are independent from each other as $h',s',a'$ vary.
In particular, $\hatP_{\kappa_{\ell,j,h}(s,a), h}(\,\cdot \mid s,a)$ is distributed as the empirical
frequency vector of $2^{\ell-1}$
i.i.d.\ random variables $\sigma_{h,s,a,j}$, with $j \in [2^{\ell-1}]$.
In addition, Bellman's equations (see the beginning of Section~\ref{sec:defalgo}) show that
on the events $\cC'$ considered to apply optional skipping,
$\hatV_{\kappa_{\ell,j,h}(s,a),h+1}$ only depends on the
$\pi_{\kappa_{\ell,j,h}(s,a),h'}$ with $h' \geq h+1$, on the $\hatP_{\kappa_{\ell,j,h}(s,a), h'}(\,\cdot \mid s',a')$
with $h' \geq h+1$, and on state-action pairs relative to stages $h' \geq h+1$.
Given the form of the adversarial learning strategy used, we conclude that
on the events $\cC'$ considered to apply optional skipping, all the
$\hatV_{\kappa_{\ell,j,h}(s,a),h+1}$, as $s,a$ vary, only depend on state-action pairs of stages $h' \geq h+1$
and are therefore independent from all the $\hatP_{\kappa_{\ell,j,h}(s',a'), h}$, as $s',a'$ vary.

Put differently, optional skipping entails here that for all $\epsilon > 0$,
\begin{align}
\nonumber
& \P \!\left( \Biggl\{ \sum_{(s,a) : \kappa_{\ell,j,h}(s,a) \leq T} \bigl( \hatP_{\kappa_{\ell,j,h}(s,a), h} - P_h \bigr)
\bullet \hatV_{\kappa_{\ell,j,h}(s,a),h+1}(s,a) > \epsilon \Biggr\} \cap \cC' \right) \\
\label{eq:C-2}
\leq \ &
\P \Biggl\{ \sum_{(s,a) : \kappa_{\ell,j,h}(s,a) \leq T} \,\,
\frac{1}{2^{\ell-1}} \sum_{j \in [2^{\ell-1}]} \bigl( \tV_{s,a,h+1}(\sigma_{h,s,a,j}) -
P_h \tV_{s,a,h+1}(s,a) \bigr) > \epsilon \Biggr\},
\end{align}
for some $\tV_{s,a,h+1}$ independent from all the $\sigma_{h,s,a,j}$ as $s,a,j$ vary (and $h$ is fixed).
We recall that the $\sigma_{h,s,a,j}$ are independent from each other as $s,a,j$ vary (and $h$ is fixed).

By the independencies noted above, and by boundedness of the values functions in $[0,(H-h)(H+1)] \subseteq [0,H^2]$,
the Hoeffding--Azuma inequality guarantees that for all $\delta' \in (0,1)$,
\begin{multline}
\label{eq:C-3}
\P \! \left\{ \sum_{(s,a) : \kappa_{\ell,j,h}(s,a) \leq T} \,\,
\frac{1}{2^{\ell-1}} \sum_{j \in [2^{\ell-1}]} \bigl( \tV_{s,a,h+1}(\sigma_{h,s,a,j}) -
P_h \tV_{s,a,h+1}(s,a) \bigr) \right. \\
\left. > \sqrt{\frac{1}{2 \times 2^{\ell-1}}
\sum_{(s,a) : \kappa_{\ell,j,h}(s,a) \leq T} H^4 \ln\biggl(\frac{1}{\delta'}\biggr)} \right\} = \delta'\,.
\end{multline}

We summarize what we proved so far.
Denoting by
\[
\Delta_{T,\ell,j} \eqdef
\sum_{t=1}^T \sum_{(s,a)} \bigl( \hatP_{t, h} - P_h \bigr) \bullet V_{h+1}^{\bpi_t, \br_t + \bb_t, \hatbP_{t}}(s,a)
\, \ind\bigl\{ (s_{t,h}, a_{t,h}) = (s,a) \bigr\} \, \ind\bigl\{ n_{t,h}(s,a) = 2^{\ell-1} + j \bigr\}
\]
the target quantity, and by
\[
\cC \eqdef \bigcap_{(s',a') \in \cS \times \cA} \left(
\bigl\{ \tau_{\ell,j,h}(s',a') = \kappa_{\ell,j,h}(s',a') \bigr\} \bigcap_{t' =\in [T]} \bigcap_{h' \in [H-1]}
\bigl\{ \ell_{t',h'}(s',a') = k_{t',h'}(s',a') \bigr\} \right)
\]
the event associated with the values fixed,
the bounds~\eqref{eq:C-1}--\eqref{eq:C-2}--\eqref{eq:C-3} show that for all $\delta' \in (0,1)$,
\begin{align}
\nonumber
& \P \! \left( \Biggl\{ \Delta_{T,\ell,j} > \sqrt{2H^4 \, \frac{1}{2^{\ell-1}}
\sum_{(s,a)} \ind\bigl\{ n_{T,h}(s,a) \geq 2^{\ell-1} + j \bigr\} \ln\biggl(\frac{1}{\delta'}\biggr)} \Biggr\} \cap \cC
\right)\\
\label{eq:C-Part1}
= \ &
\P \! \left( \Biggl\{ \Delta_{T,\ell,j} > \sqrt{\frac{1}{2^{\ell}}
\sum_{(s,a) : \kappa_{\ell,j,h}(s,a) \leq T} H^2 \ln\biggl(\frac{1}{\delta'}\biggr)} \Biggr\} \cap \cC
\right) \leq \delta'\,.
\end{align}

\emph{Part 2: Union bound and counting the sequences.}
The proof is concluded by counting how many different sets $\cC$ may be obtained.
We do so in a rough way, that will be sufficient for our purposes.
First, we need to count the profile values~\eqref{eq:profiles}.
There are $T(H-1)$ functions $\ell_{t, h} : \cS \times \cA \to \bigl[ \ilog{2}{T} \bigr]^*$,
satisfying some monotonicity constraints, as well as some other constraints which we ignore.
The monotonicity constraints imply that for each $(s,a)$ and $h \in [H-1]$, it is sufficient to determine
the at most $\ilog{2}{T}$ time steps $t$ among $[T]$ when $\ell_{t, h}$ increases by~$1$.
Thus, there are at most
\[
(T+1)^{\ilog{2}{T}}
\]
possible sequences of values for the $\ell_{t, h}(s,a)$ as $t$ varies and $h,s,a$ are fixed.
All in all, the profile part in the number of different sets $\cC$ is smaller than
\[
\bigl( (T+1)^{\ilog{2}{T}} \bigr)^{SA(H-1)}\,.
\]

For stopping times, we need to determine, for each $(s,a)$, a single value, in a set included in
$[T] \cup \{+\infty\}$. We neglect other constraints and see that there are therefore at most $(T+1)^{SA}$ such choices.

As a conclusion, there are at most
\[
M = (T+1)^{SAH \ilog{2}{T}}
\]
different possible values for the sets $\cC$. The proof of Lemma~\ref{lm:main-termC}
is concluded by a union bound over the events \eqref{eq:C-Part1},
with $\delta' = \delta/(4HTM)$.
\end{proof}

\end{document}